\documentclass{article}

\usepackage[utf8]{inputenc}
\usepackage[margin=1in]{geometry}

\usepackage{amsmath,amsthm,amsfonts,amssymb,mathdots,array,epstopdf,mathrsfs,bm,stmaryrd,graphicx,subfigure,xcolor}
\usepackage{algpseudocode,algorithm,algorithmicx}

\usepackage[T1]{fontenc}
\usepackage{enumerate}
\usepackage{inputenc}

\usepackage{graphicx} 
\usepackage{subfigure}

\usepackage{booktabs,balance}

\newtheorem{theorem}{Theorem}[section]

\newtheorem{lemma}[theorem]{Lemma}
\newtheorem{proposition}[theorem]{Proposition}

\usepackage{yhmath}

\usepackage[colorlinks=true,linkcolor=blue,urlcolor=blue,citecolor=purple]{hyperref}

\theoremstyle{definition}

\theoremstyle{remark}
\newtheorem{remark}{Remark}

\newcommand{\wtd}{\widetilde}
\newcommand{\wht}{\widehat}

\DeclareMathOperator{\diag}{diag}
\DeclareMathOperator{\rank}{rank}
\DeclareMathOperator{\opt}{opt}
\DeclareMathOperator{\argmin}{argmin}

\newcommand{\J}{\mathcal{J}}

\newcommand{\PO}{\Pi_{\Omega}}

\newcommand{\tol}{\texttt{tol}}
\newcommand{\T}{\rm T}
\newcommand{\R}{\mathbb{R}}
\newcommand{\ve}[1]{\mathbf{#1}}
\newcommand{\ma}[1]{\mathbf{#1}}
\newcommand{\te}[1]{\bm{\mathcal{#1}}}

\newcommand{\bsmat}{\left[\begin{smallmatrix} }
\newcommand{\esmat}{\end{smallmatrix}\right] }

\renewcommand{\vec}{\text{vec}}

\begin{document}

\title{Tensor Completion via Tensor Networks with a Tucker Wrapper}

\author{\textbf{Yunfeng Cai, \ Ping Li} \\\\
Cognitive Computing Lab\\
Baidu Research\\
  No.10 Xibeiwang East Road, Beijing 100193, China\\
  10900 NE 8th St. Bellevue, Washington 98004, USA\\
  \texttt{\{caiyunfeng,\  liping11\}@baidu.com}
}

\date{}
\maketitle

\begin{abstract}
In recent years, low-rank tensor completion (LRTC) has received considerable attention
due to its applications in image/video inpainting, hyperspectral data recovery, etc.
With different notions of tensor rank (e.g., CP, Tucker, tensor train/ring, etc.), various optimization based numerical methods are proposed to LRTC.
However, tensor network based methods have not been proposed yet.
In this paper, we propose to solve LRTC via tensor networks with a Tucker wrapper.
Here by ``Tucker wrapper'' we mean that the outermost factor matrices of the tensor network are all orthonormal.
We formulate LRTC as
a problem of solving a system of nonlinear equations,
rather than a constrained optimization problem.
A two-level alternative least square method is then employed to update the unknown factors.
The computation of the method is dominated by tensor matrix multiplications and can be efficiently performed.
Also, under proper assumptions, it is shown that with high probability, the method converges to the exact solution at a linear rate.
Numerical simulations show that the proposed algorithm is comparable with state-of-the-art methods.
\end{abstract}

\section{Introduction}\label{sec:intro}

Tensors are multi-dimensional arrays, which are generalizations of vectors and matrices.
Tensors are natural tools for the representation of high dimensional data.
For example, EEG signal is a third tensor (time $\times$ frequency $\times$ electrodes);
a color video is a fourth-order tensor (width $\times$ height $\times$ 3 $\times$ time).
Tensors and their decompositions nowadays become increasingly popular
and become fundamental tools to deal with high dimensional data.
We refer the readers to~\cite{Article:Acar_TDDE09,Article:Kolda_SR09,Artice:Cichocki_SPM15,Article:Papalexakis_TIST17,Article:Sidiropoulos_TSP17} for tensors, their decompositions and applications.

Due to the data acquisition process and/or outliers, values can be missing in data.
People have great interests in inferring the missing values (e.g., recommender system).
However, there are infinitely many ways to fill in the missing values without further assumptions.
It is commonly assumed that the high dimensional data lie in a low dimensional manifold.
(For example, in a recommendation system, it is commonly believed that users' behaviors are dictated by a few common factors.)
Upon such an assumption, people may learn the low dimensional manifold from the observed data,
then infer the missing values.
In current literature for tensor completion,
the low dimensional manifold is represented by a ``low rank'' tensor decomposition.
Tensor rank differs from matrix rank dramatically
(e.g., a real-valued tensor may have different tensor ranks over $\mathbb{R}$ and $\mathbb{C}$;
the best low rank approximation of a high order tensor may not exist),
and it's the cornerstone of all methods for LRTC (low-rank tensor completion).

Mathematically, the LRTC problem can be formulated as the following optimization problem:
\begin{align}\label{main}
\min_{\te{X}} \rank_*(\te{X}),\quad \mbox{subject to} \quad \PO(\te{X})=\PO(\te{T}),
\end{align}
where $\rank_*(\cdot)$ denotes a specific type of tensor rank, $\Omega$ stores the indices of the observed entries,
and $\PO$ picks the entries of a tensor with entries' indices in $\Omega$.
With different notations of tensor rank, various methods are proposed to solve LRTC.
First, tensor has a CP decomposition (CPD)~\cite{carroll1970analysis,harshman1970foundations,hitchcock1927expression},
and the tensor CP rank. However, the determination of the CP rank is NP-hard~\cite{Article:Hastad_JOA90}.
Thus in practice the CP rank is usually treated as a parameter that can be tuned.
Several CP rank based methods are proposed in last two decades, e.g.,
INDAFAC~\cite{Article:Tomasi_2005}, CP-WOPT~\cite{Article:Acar_2011}, BPTF~\cite{Proc:Xiong_SDM10}, STC~\cite{Proc:Krishnamurthy_NIPS13}.
Second, tensors has Tucker decomposition/high-order SVD (HOSVD)~\cite{tucker1963implications,tucker1966some}
and the Tucker/multi-linear rank.
Based on such a decomposition, many numerical methods are proposed.
To name a few, pTucker~\cite{Proc:Chu_AISTATS09}, MRTF~\cite{Proc:Karatzoglou_Recsys10},
geomCG~\cite{Article:Kressner_BNM14}, Tmac~\cite{xu2013parallel}, FaLRTC/HaLRTC~\cite{Article:Liu_PAMI13}, Square Deal~\cite{Proc:Mu_ICML14}.
Tensors also have other decompositions and related rank definitions, e.g.,
tensor train (TT)~\cite{Article:Oseledets_SJSC11} and TT rank~\cite{Proc:Imaizumi_NIPS17},
tensor ring (TR) and TR rank~\cite{zhao2016tensor},
t-SVD and tubal-rank~\cite{kilmer2011factorization},
the related methods includes~\cite{Proc:Bengua_TIP17,ding2020tensor,Article:Grasedyck_SJSC15,Article:Ko_TIP20,Proc:Yuan_ICONIP17,Article:Yuan_SPIC19,Proc:Zhang_CVPR14}, etc.
Among the various numerical methods for LRTC, some of them have theoretical guarantees for exact recovery under proper assumptions, e.g., Jain and Oh~\cite{Proc:Jain_NIPS14}, Liu and Moitra~\cite{Proc:Liu_NeurIPS20},
Yuan and Zhang~\cite{Article:Yuan_IT17},
Mu {\em el al.}~\cite{Proc:Mu_ICML14}, Xia and Yuan~\cite{Article:Xia_FCSM19}, Zhang and Aeron~\cite{Article:Zhang_TSP17}.
We refer the readers to a recent survey~\cite{Article:Song_TKDD19} for a comprehensive overview of LRTC.

\begin{figure}[!hbt]
\centering
\subfigure[CP]{\includegraphics[width=2in]{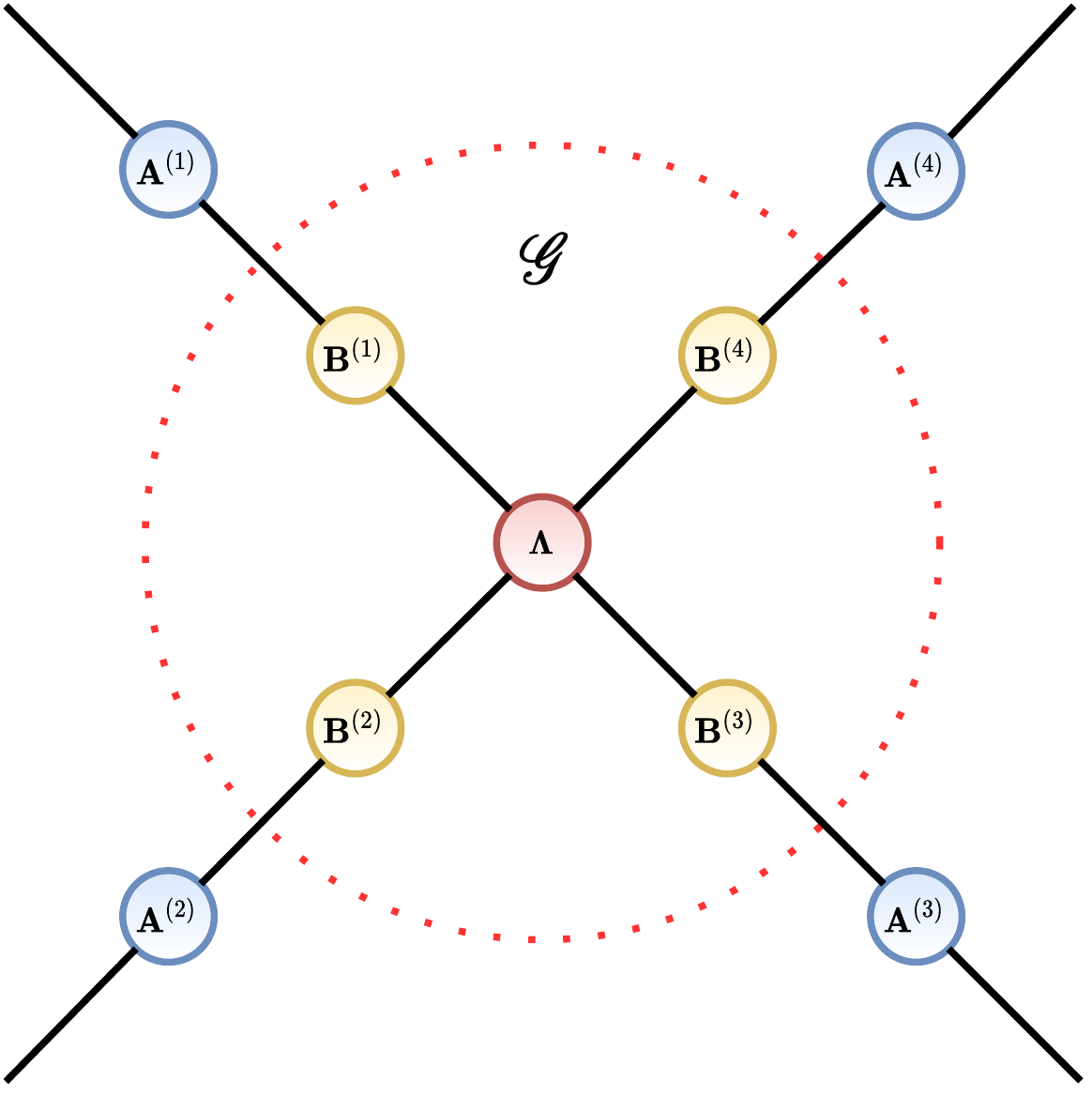}}\hspace{0.3in}
\subfigure[HT]{\includegraphics[width=2in]{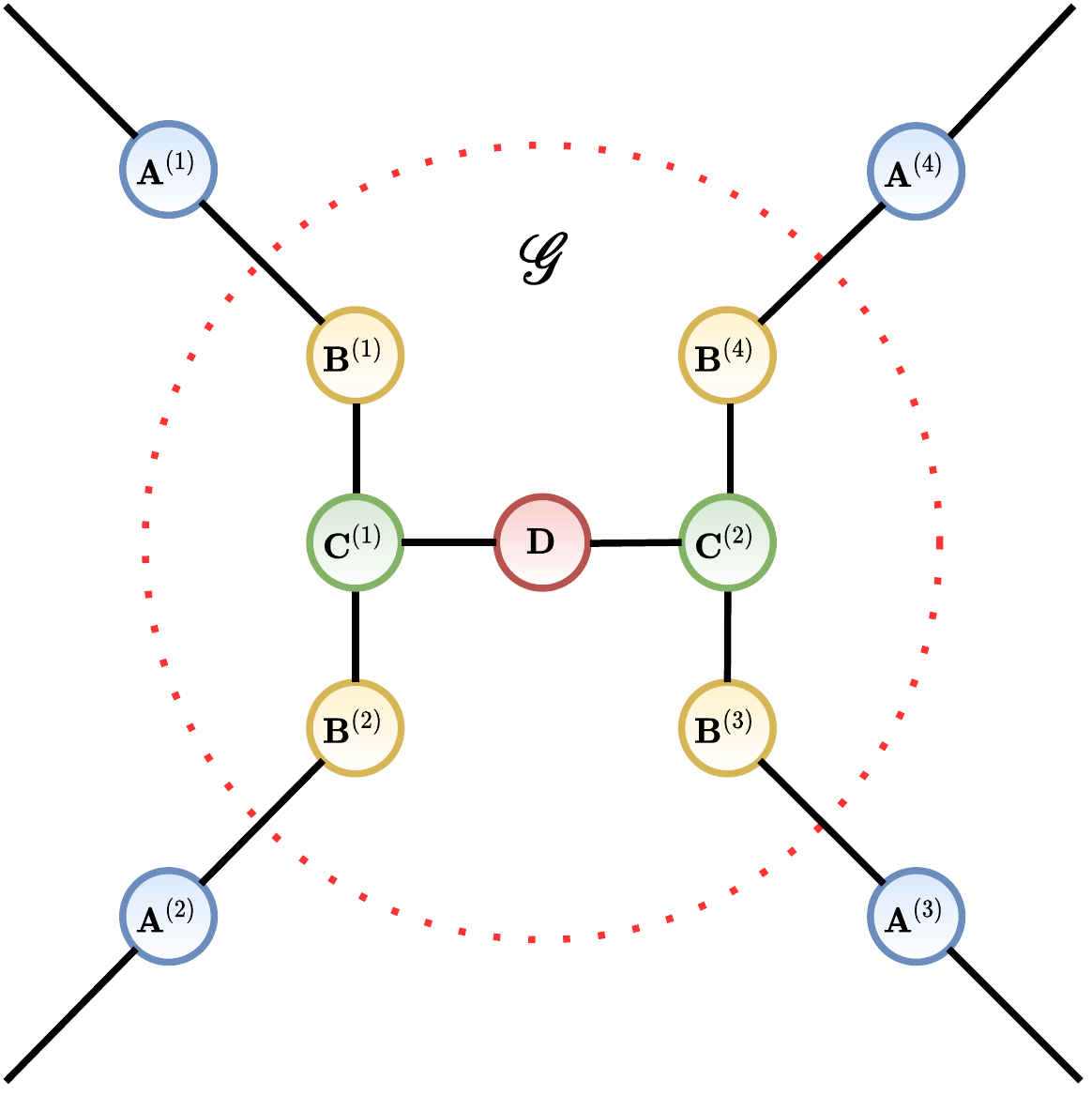}}

\subfigure[TT]{\includegraphics[width=2in]{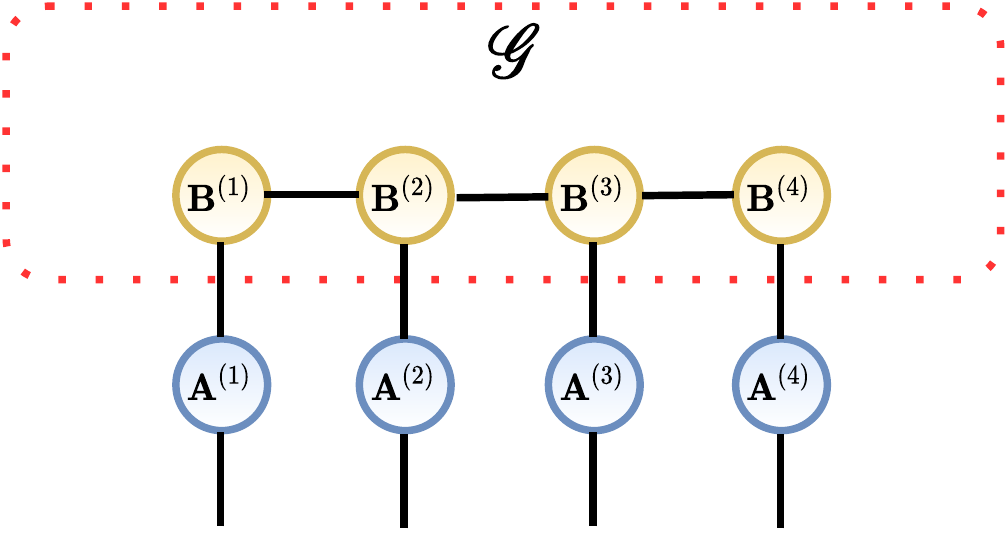}}\hspace{0.3in}
\subfigure[TR]{\includegraphics[width=2in]{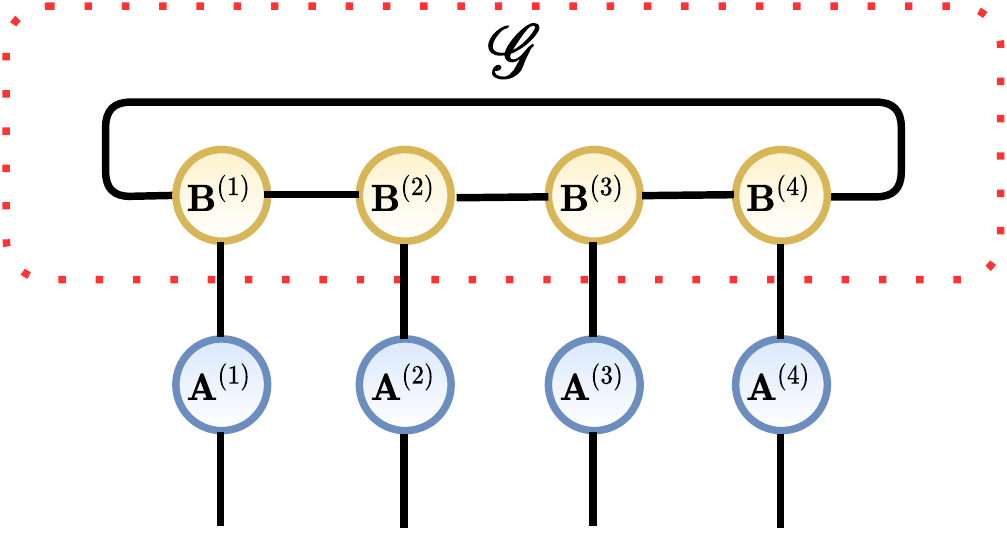}}
\vspace{-0.1in}
\caption{A fourth-order tensor network with a Tucker wrapper: canonical polyadic (CP), Hierarchical Tucker (HT), tensor train (TT) and tensor ring (TR), the weights of all edges are omitted for simplicity.}
\label{fig_1}
\end{figure}

To the best of the authors' knowledge,
 existing (popular) methods for LRTC are all optimization based.
And perhaps, due to the difficulty for finding an appropriate $\rank_*(\cdot)$ for tensor networks (TN), (which plays the role of the nuclear norm for matrix), optimization based methods have not been proposed for LRTC based on TN.
In this paper, we adopt a TN model with a Tucker wrapper,
or equivalently, a Tucker/HOSVD model with the core tensor having a TN structure, see Figure~\ref{fig_1} for illustrations, where the tensor diagram notation~\cite{bridgeman2017hand} is adopted. Unfamiliar readers may refer to Figure~\ref{fig_2} first.
What's more, we formulate the LRTC problem as
a problem solving a system of nonlinear equations (SNLE),
rather than a constrained optimization problem as in \eqref{main}.
Then we propose a two-level alternative least square method to solve the SNLE.
Under a ``low rank'' assumption, that is, the outermost factor matrices are low rank and the core tensor can be represented by a TN with a small number of parameters,
we show that with high probability, the method converges to the exact solution at a linear rate.
Finally, numerical simulations show the merits of the method.

The rest of this paper is organized as follows.
In Section~\ref{sec:pre}, we present some preliminary results.
In Section~\ref{sec:alg}, we formulate the LRTC problem as a problem of solving a system of nonlinear equations
and present an algorithm to solve it.
The convergence analysis of the algorithm is then presented in Section~\ref{sec:convergence}.
Numerical simulations are provided in Section~\ref{sec:numer}.
Concluding remarks are given in Section~\ref{sec:conclusion}.

\vspace{0.1in}
\noindent{\bf Notations}\; In this paper, we use lowercase letters to denote scalars (e.g., $a,b$),
boldface lowercase letters to denote column vectors (e.g., $\ve{a, b}$),
boldface uppercase letters to denote matrices (e.g., $\ma{A, B}$),
and boldface calligraphic letters to denote tensors (e.g., $\te{A, B}$).
The symbol $\otimes$ denotes the Kronecker product.
The operation $\vec(\ma{X})$ denotes the vectorization of the matrix $\ma{X}$ formed by stacking the columns of $\ma{X}$ into a single column vector.
The identity matrix of order $N$ is denoted by $\ma{I}_{N}$.
For a (rectangular) matrix $\ma{A}\in\R^{m\times n}$, its singular values are denoted by
$\sigma_1(\ma{A})\ge \dots\ge \sigma_{\min\{m, n\}}(\ma{A})\ge 0$,
and $\sigma_{\min\{m, n\}}(\ma{A})$ is usually denoted by $\sigma_{\min}(\ma{A})$.
The rank of $\ma{A}$ is denoted by $\rank(\ma{A})$.
The 2-norm and Frobenius norm are denoted by $\|\cdot\|$ and $\|\cdot\|_F$, respectively.
The range space of $\ma{A}$, which is the subspace spanned by the column vectors of $\ma{A}$, is denoted by $\mathcal{R}(A)$.

\section{Preliminary}\label{sec:pre}

In this section, we present some notations and preliminary results for facilitating of our following discussions.

\vspace{0.1in}
\noindent{\bf Canonical Angles}\;
Let $\mathcal{X}, \mathcal{Y}$ be two $k$-dimensional subspaces of $\R^n$.
Let $\ma{X}, \ma{Y} \in\R^{n\times k}$ be the orthonormal basis matrices of
$\mathcal{X}$ and $\mathcal{Y}$, respectively, i.e.,
\[
\mathcal{R}(\ma{X}) = \mathcal{X},\ \ma{X}^{\T}\ma{X}=\ma{I}_k,
\ \mbox{ and } \
\mathcal{R}(\ma{Y}) = \mathcal{Y},\ \ma{Y}^{\T}\ma{Y}=\ma{I}_k.
\]
Denote $\omega_j$ for $1\le j\le k$ the singular values of $\ma{Y}^{\T}\ma{X}$
in ascending order, i.e., $\omega_1\le\dots\le\omega_k$.
The $k$ {\em canonical angles $\theta_j(\mathcal{X},\mathcal{Y})$
between $\mathcal{X}$ and $\mathcal{Y}$ } are defined by
\begin{equation*}
0\le\theta_j(\mathcal{X},\mathcal{Y}):=\arccos\omega_j\le\frac {\pi}2,\quad\mbox{for $1\le j\le k$}.
\end{equation*}
They are in descending order, i.e., $\theta_1(\mathcal{X},\mathcal{Y})\ge\cdots\ge\theta_k(\mathcal{X},\mathcal{Y})$.
Set
\begin{equation*}
\Theta(\mathcal{X},\mathcal{Y})=\diag(\theta_1(\mathcal{X},\mathcal{Y}),\ldots,\theta_k(\mathcal{X},\mathcal{Y})).
\end{equation*}
Notice that if $k=1$, the canonical angle is nothing but the
angle between two vectors.
In what follows, we sometimes place a vector or matrix in one or both
arguments of $\theta_j(\,\cdot\,,\,\cdot\,)$ and $\Theta(\,\cdot\,,\,\cdot\,)$ with the meaning that
it is about the subspace spanned by the vector or the column vectors of the matrix argument.


\vspace{0.1in}
\noindent{\bf Modal Unfolding}\;
Given a tensor $\te{T}\in\R^{I_1\times\dots\times I_N}$, its mode-$n$ unfolding is an $I_n$-by-$\prod_{k\ne n}I_k$ matrix, its columns are the mode-$n$ fibers of $\te{T}$, denoted by $\te{T}_{(n)}$.

\vspace{0.1in}
\noindent{\bf Modal Product} \
Given a tensor $\te{T}\in\R^{I_1\times\dots\times I_N}$ and a matrix $\ma{U}\in\R^{J\times I_n}$, the mode-$n$ product of $\te{T}$ and $\ma{U}$ is an $I_1$-by-$\cdots$-by-$I_{n-1}$-by-$J$-by-$I_{n+1}$-by-$\cdots$-by-$I_N$ tensor, denoted by $\te{T}\times _n \ma{U}$. Let $\te{S}=\te{T}\times_n\ma{U}$. The mode-$n$ product can be defined via modal unfolding as
$\te{S}_{(n)}=\ma{U}\te{T}_{(n)}$.

The operation $\te{S}=\te{T}\times_1\ma{A}_1\dots\times_N \ma{A}_N$ is usually denoted as
$\llbracket \te{T}; \ma{A}_1,\dots, \ma{A}_N\rrbracket$. Furthermore, in such case, it holds the following important equality:
\begin{align*}
\te{S}_{(n)}= \ma{A}_n \te{T}_{(n)} (\ma{A}_N\otimes \dots \ma{A}_{n+1} \otimes \ma{A}_n \otimes \dots\otimes \ma{A}_1)^{\T}.
\end{align*}

\vspace{0.1in}
\noindent{\bf High Order SVD (HOSVD)}\; Given an $N$th tensor $\te{T}\in\R^{I_1\times\dots\times I_N}$. Let $\te{T}_n=\ma{U}_n\ma{\Sigma}_n\ma{V}_n^{\T}$ for $n=1,\dots, N$ be the SVDs of the modal unfoldings of $\te{T}$. Then the HOSVD of $\te{T}$ can be given by
\begin{align*}
\te{T}=\te{S}\times_1\ma{U}_1\dots\times_N\ma{U}_N=\llbracket \te{S}; \ma{U}_1,\dots, \ma{U}_N\rrbracket,
\end{align*}
where $\te{S}=\llbracket \te{S}; \ma{U}_1^{\T},\dots, \ma{U}_N^{\T}\rrbracket$ is the core tensor.
Furthermore, if the SVDs of $\te{T}_{(n)}$'s are economic (zero singular values of $\ma{\Sigma}_n$ are removed, the corresponding left and right singular vectors are also removed from the column vectors of $\ma{U}_n$ and $\ma{V}_n$, respectively),
then $\te{S}$ will be $r_1$-by-$\cdots$-by-$r_N$, where $r_n=\rank(\te{T}_{(n)})$.
In such case, the HOSVD will be referred to as an economic HOSVD.
The vector $(r_1,\dots,r_N)$ is the multi-linear rank of $\te{T}$, denoted by $\rank_n(\te{T})$.
In addition, if small singular values of $\ma{\Sigma}_n$ are also removed, then one would obtain the truncated HOSVD.

\vspace{0.1in}
\noindent{\bf Tucker Decomposition}\;
Given an $N$th tensor $\te{T}\in\R^{I_1\times\dots\times I_N}$ and a vector $\mathbf{r}\le \rank_n(\te{T})$ with inequality in at least one component.
The Tucker decomposition tries to find a tensor $\te{X}$
such that
\begin{equation}\label{tucker}
\min_{\te{X}}\|\te{T}-\te{X}\|_F, \; \mbox{s.t.}\; \rank_n(\te{X})=\mathbf{r}.
\end{equation}
The truncated HOSVD does not solve \eqref{tucker}, but provides a good approximation.

\vspace{0.1in}
\noindent{\bf Tensor Network and Graph}\;
A tensor network aims to represent a high order tensor into a set of lower order (usually 2 or 3) tensors, which are connected sparsely.
In such a way, the curse of dimensionality can be greatly alleviated or even  avoided.
Tensor diagram notation is a simple yet effective way to represent tensor networks,
in which a node represents a tensor,
an edge between two nodes indicates a contraction between the two connected node tensors in the associated pair of modes,
each outgoing edge represents a mode, and the weight above it indicates the size of the mode.
For example, in Figure~\ref{fig_2}, the node with none/one/two/three weighted edges stands for a scalar/a length $I$ vector/an $I_1$-by-$I_2$ matrix/an $I_1$-by-$I_2$-by-$I_3$ tensor,
two nodes connected by an edge with weight $I$ stands for the inner product of two length $I$ vectors,
two nodes with one outgoing edge each and one common edge stand for the multiplication between an $I_1$-by-$I_2$ matrix and an $I_2$-by-$I_3$ matrix.

\begin{figure}[!hbt]
\centering
\includegraphics[width=3.0in]{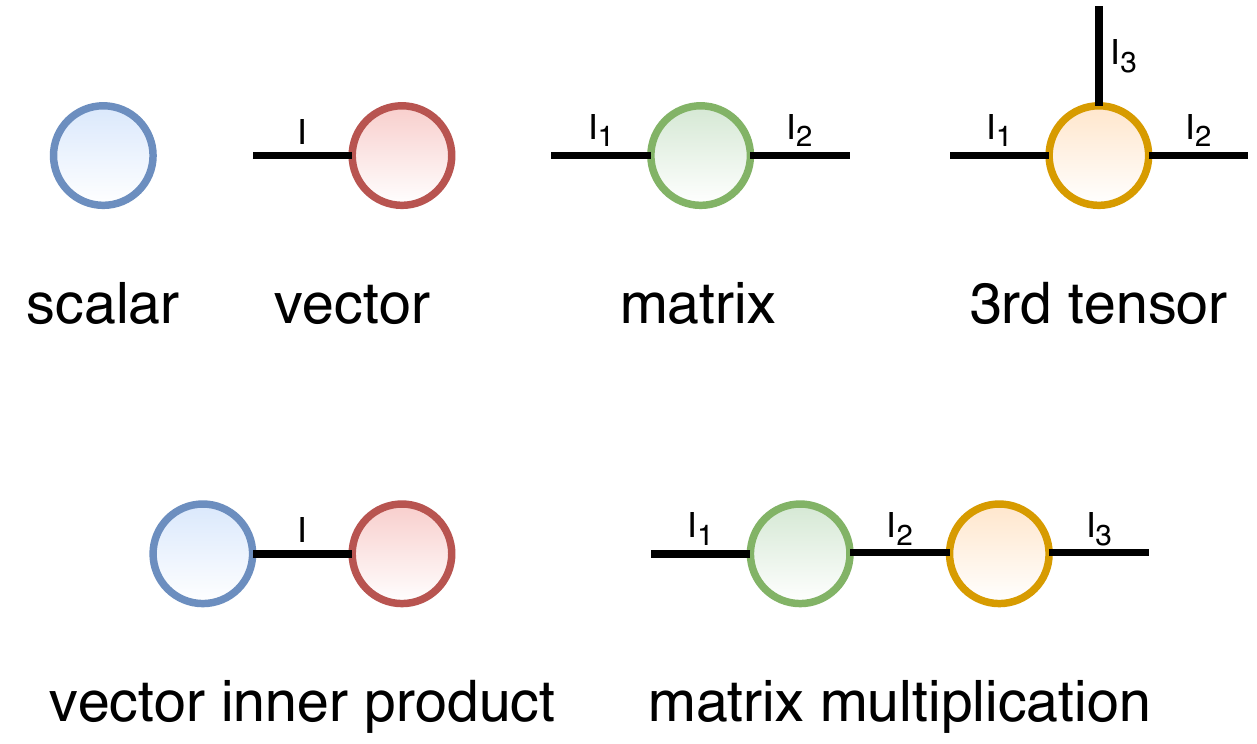}
\caption{Basic symbols for tensor diagram notations}
\label{fig_2}
\end{figure}

Simply speaking, a graph is a structure that consists of nodes that may or may not be connected with one another.
A graph is undirected if the edge is undirected.
A graph is weighted if each of its edges is assigned with a number (weight).

Now we may embed a tensor
$\te{T}\in\R^{I_1\times\dots\times I_N}$ into a weighted and undirected graph $\texttt{G}$
as follows: for each mode, pick a node from $\texttt{G}$ and assign the node an outgoing edge with weight $I_n$.
Denote the resultant graph-like structure as $\texttt{G}^+$.
We can uniquely construct a TN for a tensor $\te{T}$ from $\texttt{G}^+$
since $\texttt{G}^+$ is essentially a tensor diagram.
So, for any tensor with a TN decomposition, we can rewrite it as:
\begin{equation}\label{tng}
\te{T}\triangleq\te{T}(\texttt{G}^+(\ve{w},\ve{d}),\mathscr{B})
\end{equation}
where $\texttt{G}^+$ is a tensor diagram constructed
from a graph $\texttt{G}$,
$\ve{w}$ is the weight vector for $\texttt{G}$ (the edges of $\texttt{G}$ need to be numbered),
$\ve{d}$ is the weight vector of all outgoing edges, i.e., the dimension of $\te{T}$,
$\mathscr{B}$ is the collection of the node tensors in $\texttt{G}$. For example, in Figure~\ref{fig_3}, $\te{T}$ is a tensor of dimension $10\times20\times30\times40$,
$\ve{w}=(10,8,12,18,20,15)$, $\ve{d}=(10,20,30,40)$,
$\mathscr{B}$ consists of 5 matrices and 2 order-3 tensors.
Note here that modes 3 and 4 are assigned to the same node.

\begin{figure}[!hbt]
\centering
\includegraphics[width=3.3in]{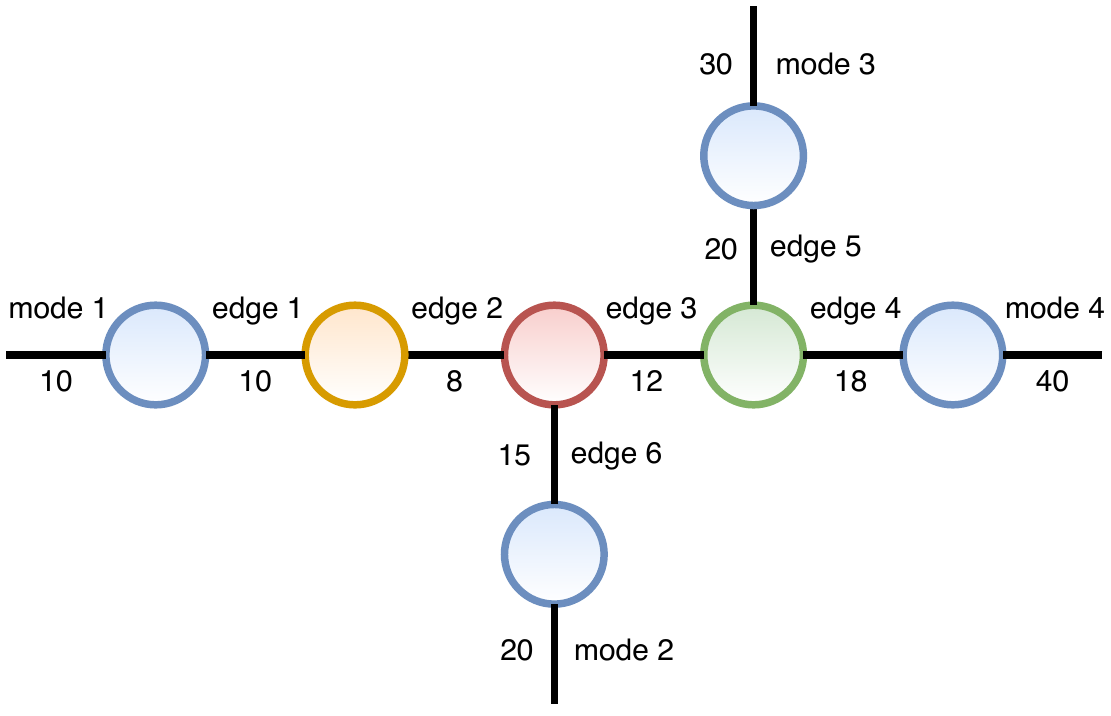}
\caption{An example for $\texttt{G}^{+}(\ve{w},\ve{d})$}
\label{fig_3}
\end{figure}

\section{Algorithm}\label{sec:alg}

In this section, we first motivate the method in Section~\ref{sec:snle},
then summarize it in Section~\ref{sec:alg2}.

\subsection{Solving a System of Nonlinear Equations}\label{sec:snle}

Let $\te{X}$ have the following TN decomposition with a Tucker wrapper:
\begin{align}\label{hosvd}
\te{X}=
\llbracket \te{G}(\texttt{G}^+(\ve{w},\ve{d}),\mathscr{B}); \ma{A}^{(1)},\dots,\ma{A}^{(N)}\rrbracket,
\end{align}
where $\te{G}=\te{G}(\texttt{G}^+(\ve{w},\ve{d}),\mathscr{B})\in\R^{r_1\times\dots\times r_N}$ is the core tensor with $\rank_n(\te{G})=(r_1,\dots,r_N)$,
$\texttt{G}^+(\ve{w},\ve{d})$ and $\mathscr{B}$ are similarly defined as in \eqref{tng},
$\ma{A}^{(n)}\in\R^{I_n\times r_n}$ has orthonormal columns for all $n$.
The reason why we add a Tucker wrapper is that the size of the core tensor $\te{G}$ is smaller (in practice much smaller) than that of $\te{T}$,
by instinct, the structure of TNs for $\te{G}$ should be simpler than that for $\te{T}$, and the parameters of TNs for $\te{G}$ should be much less than that for $\te{T}$.

Given $\texttt{G}^+(\ve{w},\ve{d})$ and assume $\te{X}$ has the decomposition \eqref{hosvd}, then
\begin{equation}\label{snle}
\PO(\te{X})=\PO(\te{T})
\end{equation}
is a system of nonlinear equations (SNLE) in $\ma{A}^{(n)}$'s and the node tensors in $\mathscr{B}$.
One may use the alternative least square (ALS) method to update $\ma{A}^{(n)}$'s and the node tensors in $\mathscr{B}$
in some prescribed order.
That is, fix all $\ma{A}^{(n)}$'s and all node tensors in $\mathscr{B}$ except one,
update the one via minimizing $\|\PO(\te{X}) - \PO(\te{T})\|_F$;
update all $\ma{A}^{(n)}$'s and the node tensors in $\mathscr{B}$ in some prescribed order until convergence.

In order to update the factor matrices $\ma{A}^{(n)}$'s and the node tensors in $\mathscr{B}$,
we need to know $\texttt{G}^+(\ve{w},\ve{d})$, i.e., the tensor diagram (including the weights of all edges).
Without knowing any information of $\texttt{G}^+(\ve{w},\ve{d})$, the problem NP-hard.
To the best of the authors' knowledge, such a problem is only recently discussed for tensor network decomposition~\cite{Proc:Li_ICML20},
where a genetic meta-algorithm was used, and only the simplest case was treated.

In this paper, to simplify the problem,
we assume $\texttt{G}^+$ and $\ve{w}$ to be known,
and we only have a good initial guess for $\ve{d}$.
Furthermore, for the sake of a convergence guarantee,
we propose to solve the SNLE via a two-level ALS  method.
Simply put,
in the first level, we update $\ma{A}^{(n)}$ for $n=1,2,\cdots,N$;
in the second level, we update the node tensors in $\mathscr{B}$ until the core tensor $\te{G}$ converges.
Next, we show how to update $\ma{A}^{(n)}$'s, the node tensors in $\mathscr{B}$ and the weight vector $\ve{d}$ in detail.

\paragraph{Updating the factor matrices $\ma{A}^{(n)}$'s}
Let $\te{G}_{t-1}$ be the current estimation for the core tensor.
Assume that we update $\ma{A}^{(n)}$ for $n=1,\dots,N$ one by one,
and we are updating $\ma{A}^{(n)}$ with $\ma{A}^{(1)}$, $\dots$, $\ma{A}^{(n-1)}$ updated,
$\ma{A}^{(n+1)}$, $\dots$, $\ma{A}^{(N)}$ to be updated.
Then we denote the current estimations for $\ma{A}^{(1)}$, $\dots$, $\ma{A}^{(n-1)}$ and $\ma{A}^{(n+1)}$, $\dots$, $\ma{A}^{(N)}$
by $\ma{A}_t^{(1)}$, $\dots$, $\ma{A}_t^{(n-1)}$ and $\ma{A}_{t-1}^{(n+1)}$, $\dots$, $\ma{A}_{t-1}^{(N)}$, respectively.
In order to update $\ma{A}^{(n)}$, we need to solve the following optimization problem:
\begin{align}
\|\PO(\llbracket \te{G}_{t-1}; \ma{A}_t^{(1)},\dots,\ma{A}_t^{(n-1)},\ma{X},\ma{A}_{t-1}^{(n+1)},\dots,\ma{A}_{t-1}^{(N)}\rrbracket )
-\PO(\te{T})\|_F=\min. \label{anx}
\end{align}
Notice that \eqref{anx} is nothing but a linear least square (LLS) problem,
since theoretically it can be rewritten in the form $\|\ma{A}\ve{x} - \ve{b}\|$.

At first glance, the LLS problem has $|\Omega|$ equations and $I_n r_n$ unknowns.
A closer examination indicates that the LLS problem can be decoupled into $I_n$ independent smaller LLS problems,
that is, to solve $\ma{X}$ row-wise.
The LLS problem with the $i$th row of $\ma{X}$ as its unknowns,
has $\omega_{i,n}\triangleq|\{(i_1,\dots,i_N)\in\Omega\;|\; i_n=i\}|$ equations.
For simplicity, let us assume $I_1=\dots=I_N=d$, $r_1=\dots=r_N=r$, and each entry of $\te{T}$ is observed independently with probability $p$.
Then we have $\omega_{i,n}\approx p d^{N-1}$, and $\ma{X}\in\R^{d\times r}$.
So, when solving the LLS problem with a standard direct solver (say the normal equation method),
the computational cost for computing one row of $\ma{X}$ is $O(p d^{N-1}r^2)$.
And updating all $\ma{A}_t^{(n)}$ once requires $O(Npd^N r^2)$,
which are very expensive when $d$ and $N$ are large.
To reduce the computation cost, we may compute a sub-optimal solution instead. For example, we may randomly sample $O(r)$ rows from the coefficient matrix of the LLS problem,
then a sub-optimal solution can be obtained from the sampled LLS problem. And the computational cost for updating all $\ma{A}^{(n)}$'s can be reduced from $O(Npd^N r^2)$ into $O(Ndr^3)$,
which scales linearly with respect to both $N$ and $d$,
and hence, suitable for large scale problems.

In practice, we don't need to compute the exact solution to \eqref{anx},
an approximation is usually sufficient.
And to get an approximate solution,
one may use the iterative method LSQR~\cite{Article:Paige_TMS82} to solve $\ma{X}$ all at once.
In each iteration of LSQR, two matrix vector multiplications (MVPs) are needed, namely, $\ma{A}\ve{v}$ and $\ma{A}^{\T}\ve{v}$.
For the LLS problem \eqref{anx},
the operation $\ma{A}\ve{v}$ amounts to the tensor modal product: $\PO(\te{S}\times_n \ma{V})$,
where $\te{S}=\llbracket \te{G}_{t-1}; \ma{A}_t^{(1)},\dots,\ma{A}_t^{(n-1)},\ma{I},\ma{A}_{t-1}^{(n+1)},\dots,\ma{A}_{t-1}^{(N)}\rrbracket$;
and the operation $\ma{A}^{\T}\ve{v}$ amounts to the multiplication between a sparse matrix and a matrix:
$[\Pi_{\Omega}(\te{V})]_{(n)}\te{S}_{(n)}$.
As long as these two MVPs can be efficiently computed,
LSQR is preferred, especially for large scale problems.
Also, note that the sub-sample idea can also be applied to reduce the computational cost.

\paragraph{Updating the node tensors in $\mathscr{B}$}
To update the node tensors in $\mathscr{B}$, we need to know $\texttt{G}^+(\ve{w},\ve{d})$.
For the ease of illustration, we assume that
$\te{G}(\texttt{G}^+(\ve{w},\ve{d}),\mathscr{B})$ is in the CPD form,
i.e.,
\begin{equation}\label{gcpd}
\te{G}(\texttt{G}^+(\ve{w},\ve{d}),\mathscr{B})=
\llbracket \bm{\Lambda}; \ma{B}^{(1)},\dots,\ma{B}^{(N)}\rrbracket,
\end{equation}
where $\mathscr{B}=\{\bm{\Lambda},\ma{B}^{(1)},\dots,\ma{B}^{(N)}\}$,
$\bm{\Lambda}$ is a diagonal tensor, $\ma{B}^{(n)}\in\R^{r_n\times r}$ for $n=1,\dots,N$,
$\ve{w}=(r,\dots,r)$, $\ve{d}=(r_1,\dots,r_N)$.
See the upper left plot in Figure~\ref{fig_1} for an illustration.

Let the current estimations for $\ma{A}^{(1)}$, $\dots$, $\ma{A}^{(N)}$ be $\ma{A}_t^{(1)}$, $\dots$, $\ma{A}_t^{(N)}$, respectively.
We can update the node tensors in $\mathscr{B}$, that is, $\bm{\Lambda},\ma{B}^{(1)},\dots,\ma{B}^{(N)}$,
by solving
\begin{equation}\label{parb}
\|\PO(\llbracket \te{G}; \ma{A}_t^{(1)},\dots,\ma{A}_{t}^{(N)}\rrbracket)
-\PO(\te{T})\|_F=\min,
\end{equation}
where $\te{G}$ is given by \eqref{gcpd}.
Again, the ALS method can be used.
Similar to the ALS method for computing CPD, we may update $\ma{B}^{(n)}$'s as follows:
fix all $\ma{B}^{(n)}$'s but one (say $\ma{B}^{(1)}$), then \eqref{parb} becomes an LLS problem.
Similar to the way we find an approximate solution for $\ma{A}^{(n)}$'s ($\ma{B}^{(n)}$'s can no loner be computed row by row),
we can obtain a new estimation for $\ma{B}^{(1)}$, denoted by $\wht{\ma{B}}^{(1)}$.
Let the length of the $i$th column vector of $\wht{\ma{B}}^{(1)}$ be $\gamma_i$, for $i=1,\dots,r_1$.
Then we set $\ma{B}^{(1)}=\wht{\ma{B}}^{(1)}\bm{\Gamma}^{-1}$,
$\bm{\Lambda}=\bm{\Lambda}\times_1\bm{\Gamma}$,
where $\bm{\Gamma}=\diag(\gamma_1,\dots,\gamma_{r_1})$.
Such a normalization step makes the all columns of $\ma{B}^{(n)}$'s be of unit length,
which is commonly adopted in the ALS method for computing CPD.
We will perform the ALS method for updating the node tensors in $\mathscr{B}$ until the core tensor $\te{G}$ converges.
\footnote{Since the TN decomposition of tensors has some natural indeterminacies (scaling, permutation, etc.),
we should not expect all node tensors converge individually.}

For general $\te{G}(\texttt{G}^+(\ve{w},\ve{d}),\mathscr{B})$, the ALS method can also be used to update the node tensors in $\mathscr{B}$.
And the iteration continues until $\te{G}$ converges.

\paragraph{Updating the weight vector $\ve{d}$}
Recall that $\ve{d}$ is nothing but the multi-linear rank of $\te{T}$.
We recommend an over-estimated initial guess to begin with.
In each iteration, after updating $\te{G}$,
we expect to observe rank deficiency in the modal unfolding matrices of $\te{G}$.
So, when small singular values occur in $\te{G}_{(n)}$,
we may remove them and update the multi-linear rank correspondingly.
To be precise, let $\te{G}_t\in\R^{r_t^{(1)}\times\dots\times r_t^{(N)}}$ be the current estimation for the core tensor,
the economic SVD of $[\te{G}_t]_{(n)}$ be
$\te{G}_{(n)} = \ma{U}^{(n)} \ma{\Sigma}^{(n)} \ma{V}^{(n)}$, for $n=1,\dots,N$,
where $\ma{U}^{(n)}$ is orthogonal, $\ma{V}^{(n)}$ is orthonormal, and $\ma{\Sigma}^{(n)}=\diag(\sigma_1^{(n)},\dots,\sigma_{r_t^{(n)}}^{(n)})$.
Then for all $n$, we find the smallest $\sigma_{s_n}^{(n)}$ such that $\sigma_1^{(n)}\le \kappa_n \sigma_{s_n}^{(n)}$, where $\kappa_n\ge 1$ is a prescribed number.
Next, update $r_t^{(n)}=s_n$, $\ma{U}^{(n)}=\ma{U}^{(n)}_{(:,1:s_n)}$,
$\ma{A}_t^{(n)}=\ma{A}_t^{(n)}\ma{U}^{(n)}$
for all $n$,
and update $\te{G}_t=\llbracket \te{G}_t; (\ma{U}^{(1)})^{\T},\dots,(\ma{U}^{(N)})^{\T}\rrbracket$.
After that, the size of $\te{G}_t$ becomes smaller,
the condition number of $[\te{G}_t]_{(n)}$ is no more than $\kappa_n$,
and $\ma{A}^{(n)}$'s have orthonormal columns (thus still consist of a Tucker wrapper).

\begin{remark}
In \cite{Proc:Cai_AISTATS20}, the authors used a similar idea to transform the robust matrix completion problem into a problem of solving a system of nonlinear equations.
Due to the differences between high order tensors and matrices,
the approach here differs from the aforementioned one in some obvious aspects,
e.g., the multilinear rank vs. the matrix rank, $N$ orthonormal factor matrices vs. two, etc.
Besides that, the major difference between the two approaches is that
we impose a TN structure for the core tensor.
As a result, a two-level ALS method is needed rather than
a ``one-level'' ALS method for the matrix completion problem.
The reason why we impose a TN structure is that:
first, it is expensive to compute the core tensor as a whole since
$\prod_n \rank([\te{G}]_{n})$ can be large though $\rank([\te{G}]_{n})$ is assumed to be small;
second, a TN structure of the core tensor may improve the performance of the method in practical problems, see Example~2 in Section~\ref{sec:numer}.
\end{remark}

\subsection{Algorithm Details}\label{sec:alg2}

Before we present the detailed algorithm,
we need to explain some notations.
Recall $\texttt{G}^{+}$,
for each mode, pick a node from $\texttt{G}$ and assign the node an outgoing edge with certain weight.
To be specific, assume that all nodes of $\texttt{G}$ are numbered;
For each mode-$n$, we assign it to node $k_n$;
and the $k_n$ node tensor $\te{B}^{(k_n)}$ is connected with the mode-$n$ outgoing edge through its $m_n$-mode. The detailed algorithm is summarized in Algorithm~\ref{alg:tc}.
Some implementation details follow.

\begin{algorithm}[h]
\caption{Two level ALS method (\textsc{tlals})\label{alg:tc}}
\begin{algorithmic}[1]
\Require{\\
$\PO(\te{T})$: the observed tensor; \\
$\texttt{G}^+$: a graph with $N$ outgoing edges;\\
$\ve{w}$: the weight vector for all edges in $\texttt{G}$;\\
$\ve{d}_0=(r_1^{(0)},\dots,r_N^{(0)})$: the initial weight vector for all outgoing edges, and $\ve{d}_0\ge\rank_n(\te{T})$ (entrywise);\\
$(\kappa_1,\dots,\kappa_N)$: the condition number upper bound vector;\\
$\tol$ : the tolerance.}
\Ensure{\\
$\ma{A}_t^{(n)}\in\R^{I_n\times r_{t}^{(n)}}$'s: the orthonormal factors;\\
$\te{G}_t=\te{G}(\texttt{G}^+(\ve{w},\ve{d}_t),\mathscr{B}_t)$: the core tensor,
where $\ve{d}_t=(r_t^{(1)},\dots,r_t^{(N)})$, $\mathscr{B}_t=\{\te{B}_t^{(k)}\}_{k=1}^{K}$.
Here $K$ is the number of nodes in $\texttt{G}$,
$\te{B}_t^{(k)}$'s are the estimations for the node tensors in $\texttt{G}$.
}

\State 
$t \leftarrow 0$;

\State Initialize $\ma{A}_t^{(n)}\in\R^{I_n\times r_{t}^{(n)}}$ and $\mathscr{B}_t=\{\te{B}_t^{(k)}\}_{k=1}^{K}$;
\State $\te{G}_t\leftarrow\te{G}(\texttt{G}^+(\ve{w},\ve{d}_t),\mathscr{B}_t)$;

\State $\tau_t \leftarrow\|\PO(\llbracket \te{G}_t; \ma{A}_t^{(1)},\dots,\ma{A}_t^{(N)}\rrbracket -\te{T})\|_F$;

\While{$\tau_t > \tol$}
\State $t\leftarrow t+1$;

\For{$n=1,\dots,N$}
\State Solve \eqref{anx} for $\ma{X}$;
\State Compute $[\ma{Q},\ma{R}]=\textsc{qr}(\ma{X})$;
\State $\ma{A}_t^{(n)}\leftarrow\ma{Q}$, $\te{B}_t^{(k_n)}\leftarrow\te{B}_{t-1}^{(k_n)}\times_{m_n} \ma{R}$;

\EndFor

\State Update $\mathscr{B}_t=\{\te{B}_t^{(k)}\}_{k=1}^{K}$ via ALS method;

\State $\te{G}_t\leftarrow\te{G}(\texttt{G}^+(\ve{w},\ve{d}_t),\mathscr{B}_t)$;

\For{$n=1,\dots,N$}
\State Compute $[\ma{U}^{(n)}, \ma{\Sigma}^{(n)}, \sim] = \textsc{svd}([\te{G}_t]_{(n)})$,
where $\ma{\Sigma}^{(n)}=\diag(\sigma_1^{(n)},\dots,\sigma_{r^{(n)}_{t-1}}^{(n)})$;
\State $r^{(n)}_t\leftarrow|\{j\;|\;\sigma_1^{(n)}\le\kappa_n \sigma_j^{(n)}\}|$;
\State $\ma{U}^{(n)}_t\leftarrow\ma{U}_{(:,1:r_{t}^{(n)})}$,
$\ma{A}_t^{(n)}\leftarrow\ma{A}_t^{(n)}\ma{U}^{(n)}$,
$\te{B}_t^{(k_n)}\leftarrow\te{B}_{t-1}^{(k_n)}\times_{m_n} (\ma{U}^{(n)})^{\T}$;
\EndFor

\State $\te{G}_t\leftarrow\te{G}(\texttt{G}^+(\ve{w},\ve{d}_t),\mathscr{B}_t)$;

\State $\tau_t \leftarrow\|\PO(\llbracket \te{G}_t; \ma{A}_t^{(1)},\dots,\ma{A}_t^{(N)}\rrbracket -\te{T})\|_F$.
\EndWhile
\end{algorithmic}
\end{algorithm}

\noindent{\bf Line 10}\;
Inspired by the initialization of the matrix completion problem,
we initialize the factor matrices $\ma{A}^{(n)}_0$'s and the core tensor $\te{G}_0$ via the best rank-$(r_1,\dots,r_N)$ approximation of
$\frac{\prod_n I_n}{|\Omega|}\PO(\te{T})$,
which can be computed by the ALS method~\cite{TTB_Software}.
Then the node tensors $\te{B}_0^{(k)}$'s can be initialized via minimizing $\|\te{G}(\texttt{G}^+(\ve{w},\ve{d}_t),\mathscr{B}_t)-\te{G}_0\|_F$.
A large tolerance, say $10^{-2}$, is usually sufficient, for both $\ma{A}^{(n)}_0$'s and $\te{B}_0^{(k)}$'s.

\vspace{0.1in}
\noindent{\bf Line 17}\;
Here we only need to compute an ``economic'' QR decomposition, in which $\ma{Q}$ has orthonormal columns,
$\ma{R}$ is an upper triangular square matrix.

\vspace{0.1in}
\noindent{\bf Line 20}\;
For the sake of a guaranteed convergence, we need to update the node tensors until the core tensor converges.
In practice, an approximation for the core tensor is sufficient.

\vspace{0.1in}
\noindent{\bf Line 23}\;
Here we only need to compute an ``economic'' SVD, in which only the singular values and the corresponding left singular vectors are required.

\vspace{0.1in}
\noindent{\bf Line 24}\;
The singular values are truncated such that the condition number of $[\te{G}_t]_{(n)}$ is no more than $\kappa_n$.

\vspace{0.1in}
\noindent{\bf Lines 11, 21 and 27}
Assume that the contractions for the TN can be efficiently computed.
Then it is possible to perform the algorithm without formulating the core tensor $\te{G}_t$ explicitly.

\section{Convergence Analysis}\label{sec:convergence}

In this section, we study the convergence of Algorithm~\ref{alg:tc}.
We will follow the notations in Algorithm~\ref{alg:tc} and we will also make the following assumptions.

\vspace{0.1in}
\noindent{\bf A1} Each entry of $\te{T}$ is observed independently with probability $p$.
\vspace{0.05in}\\
\noindent{\bf A2} The tensor $\te{T}$ can be factorized as in \eqref{hosvd}.
\vspace{0.05in}\\
\noindent{\bf A3} The factor matrix $\ma{A}_t^{(n)}$ satisfies an incoherence condition with parameter $\mu_n$:
$\|\ma{A}_t^{(n)}\|_{2,\infty}\le \sqrt{\frac{\mu_n r_n}{I_n}}$,
for all $n$ and $t$. 
\vspace{0.05in}\\
\noindent{\bf A4} There exist two positive constants $\gamma$, $\Gamma$ such that
\begin{equation*}
\gamma\le \frac{\min\limits_{\te{X}}\|\PO(\llbracket \te{X}; \ma{A}_t^{(1)},\dots,\ma{A}_t^{(N)}\rrbracket)-\PO(\te{T})\|_F}{\tau_t}\le \Gamma,
\end{equation*}
where $\tau_t$ is the same as in Algorithm~\ref{alg:tc}.
\vspace{0.05in}\\
\noindent{\bf A5} $r_t^{(n)}\equiv r_n$, i.e., the multi-linear rank estimations are all correctly revealed.\\

\noindent Several remarks on the assumptions follow.

\vspace{0.05in}
{\bf (a)} {\bf A1-A2} are standard for tensor completion~\cite{Article:Song_TKDD19}.

\vspace{0.05in}
{\bf (b)} Recall that solving the SNLE \eqref{snle} as described in Algorithm~\ref{alg:tc} is equivalent to the minimization of
$\|\PO(\te{X})-\PO(\te{T})\|_F^2$.
If we add a regularizer $\lambda \sum_n\|\ma{A}^{(n)}\|_F^2$ and still apply Algorithm~\ref{alg:tc} (of course, with small modifications),
the computed $\|\ma{A}_t^{(n)}\|_{2,\infty}$ becomes smaller.
The larger $\lambda$ is, the smaller $\|\ma{A}_t^{(n)}\|_{2,\infty}$ is.
So, we may declare that assumption {\bf A3} is only slightly stronger than the standard assumption that
$\ma{A}^{(n)}$ satisfies an incoherence condition.

\vspace{0.05in}
{\bf (c)} Assumption~{\bf A4} essentially requires that
$\texttt{G}^+$ together with the weight vector $\ve{w}$ have a good capacity for the representation
for the minimizer of the numerator in {\bf A4}.
Such a requirement is quite natural.
In fact, for sufficiently large $\ve{w}$, $\te{G}(\texttt{G}^+(\ve{w},\ve{d}_t,\mathscr{B}_t)$ can represent any tensor of size
$r_t^{(1)}$-by-$\cdots$-by-$r_t^{(N)}$.
And in such case, $\gamma=\Gamma=1$.

\vspace{0.05in}
{\bf (d)} From Algorithm~\ref{alg:tc}, we know that for each $n$, $r_1^{(n)}\ge r_2^{(n)}\ge \dots\ge 1$.
Thus, $r_t^{(n)}$ must converge.
We assume {\bf A5} for the ease of the convergence analysis.\\

To motivate the convergence analysis,
in Section~\ref{sec:full}, we sketch a proof for the convergence when all entries of $\te{T}$ are observed.
Then the convergence results for the partial observation case are given in Section~\ref{sec:partial}.

\subsection{Full Observation Case}\label{sec:full}

The difference between the full and partial observation cases is the LLS problems in Algorithm~\ref{alg:tc}.
Note that for an LLS problem $\|\ma{A}\ve{x} -\ve{b}\|$, 
sampling the rows uniformly yields a smaller LLS problem $\|\ma{P}_{\Omega}\ma{A}\ve{x}-\ma{P}_{\Omega} \ve{b}\|$,
where $\ma{P}_{\Omega}$ a 0-1 matrix which selects the indices in $\Omega$.
Its solution $\hat{\ve{x}}_*=(\ma{P}_{\Omega}\ma{A})^{\dagger} \ma{P}_{\Omega} \ve{b}$ is the solution to
$\|\ma{A}\ve{x} -\ve{b} - (\ma{A}\hat{\ve{x}}_* - \ve{b})\|$,
which is a perturbed LLS problem for $\|\ma{A}\ve{x} -\ve{b}\|$.
So, to motivate the convergence analysis for the partial observation case,
for the full observation case,
instead of assuming {\bf A2}, we assume $\te{T}=\llbracket \te{G}_*; \ma{A}_*^{(1)},\dots,\ma{A}_*^{(N)}\rrbracket +\te{E}$,
where $\te{E}$ is a noise tensor.



First, note that under the assumption {\bf A5}, lines 24 to 29 of Algorithm~\ref{alg:tc} can be skipped.
Next, we consider Lines 18 and 22.
Let
\begin{equation}\label{mm}
\begin{split}
\ma{M}_{t,n}&=\ma{A}_{t-1}^{(N)}\otimes\dots\otimes \ma{A}_{t-1}^{(n+1)}\otimes \ma{A}_t^{(n-1)}\otimes\dots\otimes \ma{A}_t^{(1)},\\
\ma{M}_{*,n}&=\ma{A}_*^{(N)}\otimes\dots\otimes \ma{A}_*^{(n+1)}\otimes \ma{A}_*^{(n-1)}\otimes\dots\otimes \ma{A}_*^{(1)}.
\end{split}
\end{equation}
Then \eqref{anx} with all entries observed is equivalent to
\begin{equation*}
\|\ma{X}[\te{G}_{t-1}]_{(n)}\ma{M}_{t,n}^{\T} - \te{T}_{(n)}\|=\min.
\end{equation*}
Therefore, $\ma{X}$ Algorithm~\ref{alg:tc} can be given by
\begin{align}
\ma{X}&=\te{T}_{(n)} \ma{M}_{t,n} [\te{G}_{t-1}]_{(n)}^{\dagger}\notag\\
&=\big(\ma{A}_*^{(n)}[\te{G}_*]_{(n)} \ma{M}_{*,n}^{\T} + \te{E}_{(n)}\big)\ma{M}_{t,n} [\te{G}_{t-1}]_{(n)}^{\dagger}.\label{xls}
\end{align}
When $\|\te{E}\|_F$ is small,
the right hand side of \eqref{xls} almost lies in $\mathcal{R}(\ma{A}_*^{(n)})$,
thus, we expect $\|\sin\Theta(\ma{X},\ma{A}_*^{(n)})\|$ to be small.
In particular, when $\te{E}=0$, it holds that $\|\sin\Theta(\ma{X},\ma{A}_*^{(n)})\|=0$,
i.e., updating $\ma{A}_{t-1}^{(n)}$ once will find the the range space $\mathcal{R}(\ma{A}_*^{(n)})$.
So, it is not surprising to conclude that
$\mathcal{R}(\ma{A}_t^{(n)})$ is a better approximation of $\mathcal{R}(\ma{A}_*^{(n)})$ than $\mathcal{R}(\ma{A}_{t-1}^{(n)})$
when $\te{E}$ is sufficiently small.

Similarly, we may also show that
when $\te{E}$ is sufficiently small, $\te{G}_t$ is a better approximation of $\te{G}_*$ than $\te{G}_{t-1}$.
In summary, one iteration of Algorithm~\ref{alg:tc} gives better approximations for $\ma{A}_*^{(n)}$'s and $\te{G}_*$.
In particular, when there is no noise, one iteration of algorithm~\ref{alg:tc} will return a solution of LRTC.

\subsection{Partial Observation Case}\label{sec:partial}

In this section, we present the convergence of Algorithm~1 when the entries of $\te{T}$ are partially observed.

\vspace{0.1in}

The next two lemmas establish the bridges
between the partial and full observation cases.
\begin{lemma}\label{lem:xxdif}
Denote
$J_n=\prod\limits_{k\ne n} I_k$, $J_{\min}=\min\limits_k J_k$,
$g_{\max}=\max\limits_k \|\te{T}_{(k)}\|$,
 $\sin\theta_t=\max\limits_{1\le k\le N} \|\sin\Theta(\ma{A}_t^{(k)},\ma{A}_*^{(k)})\|$. Let
\begin{align*}
\te{L}(\ma{X})&=\llbracket \te{G}_{t-1}; \ma{A}_t^{(1)},\dots,\ma{A}_t^{(n-1)},\ma{X},\ma{A}_{t-1}^{(n+1)},\dots,\ma{A}_{t-1}^{(N)}\rrbracket,\\
\ma{X}_{\opt}&=\argmin \| \te{L}_t(\ma{X}) - \te{T}\|_F,\\
\wtd{\ma{X}}_{\opt}&=\argmin \|\PO(\te{L}_t(\ma{X})) - \PO(\te{T})\|_F.
\end{align*}
Assume {\bf A1-A3} and {\bf A5}, $p\in [4p_*,0.5]$
with $p_*= \frac{10}{3} \big(\log(2\prod_{k=1}^N I_k) +5\big) \max_n\prod_{k\ne n} {\frac{\mu_{k} r_k}{I_k}}$
and $\sigma_{\min}([\te{G}_t]_{(n)})\ge g_{\min}$ for all $n$ and $t$.
Also assume
\begin{align}\label{mtheta}
\|\sin\Theta(\ma{M}_{t,n} [\te{G}_{t-1}]_{(n)}^{\T}, \ma{M}_{*,n} [\te{G}_*]_{(n)}^{\T})\|\le C\sin\theta_{t-1},
\end{align}
where $\ma{M}_{t,n}$, $\ma{M}_{*,n}$ are defined in \eqref{mm}, $C>0$ is a constant.
\footnote{To understand the constant $C$ and why assumption \eqref{mtheta} holds, please see Appendix.}
Then with probability (w.p.) $\ge 1 - 2/J_{\min}^{10+\log\alpha}$, it holds that
\begin{equation*}
\|\wtd{\ma{X}}_{\opt} - \ma{X}_{\opt}\|\le \frac{6g_{\max}(C_1+C_2)}{g_{\min}}\sqrt{\frac{\alpha}{p}} \sin\theta_{t-1},
\end{equation*}
where
$C_1=C \frac{g_{\max}}{g_{\min}} \max_n\{ \sqrt{\frac{7 \mu_n r_n }{I_n}} \prod_{k\ne n} \sqrt{\mu_k r_k}\}$, and
$C_2= C \max_n\sqrt{\frac{\mu_n r_n J_n}{I_n}} $.
\end{lemma}
Lemma~\ref{lem:xxdif} tells that
the distance between the partial observation solution $\wtd{\ma{X}}_{\opt}$ and the full observation solution $\ma{X}_{\opt}$
is upper bounded: the larger $p$ is, the smaller the distance is;
the smaller $\theta_{t-1}$ is, the smaller the distance is.

\begin{lemma}\label{lem:xxcore}
Let $\phi_t=\|\llbracket \te{X}_{\opt}; \ma{A}_t^{(1)},\dots,\ma{A}_t^{(N)}\rrbracket -\te{T}\|_F$,
$\psi_t=\|\llbracket \wtd{\te{X}}_{\opt}; \ma{A}_t^{(1)},\dots,\ma{A}_t^{(N)}\rrbracket -\te{T}\|_F$,
where
\begin{align*}
\te{X}_{\opt} &=\argmin\|\llbracket \te{X}; \ma{A}_t^{(1)},\dots,\ma{A}_t^{(N)}\rrbracket -\te{T}\|_F,\\
\wtd{\te{X}}_{\opt} &=\argmin\|\PO(\llbracket {\te{X}}; \ma{A}_t^{(1)},\dots,\ma{A}_t^{(N)}\rrbracket -\te{T})\|_F.
\end{align*}
Assume {\bf A1}, {\bf A3}, $p\in [4p_*,0.5]$
with $p_*$ being the same as in Lemma~\ref{lem:xxdif}.
Then w.p. $\ge 0.99$, it holds that
\begin{align*}
\phi_t\le \psi_t \le 3/\sqrt{2}\ \phi_t.
\end{align*}
\end{lemma}

Lemma~\ref{lem:xxcore} tells that
the partial observation solution $\wtd{\te{X}}_{\opt}$ is as good as the full observation solution $\te{X}_{\opt}$,
in term of the residual. With the help of the above two lemmas, we are able to prove our main~theorem.
\begin{theorem}\label{thm:main}
Follow the notations in Lemma~\ref{lem:xxdif}. 
Assume {\bf A1-A5}, $p\in[4p_*,0.5]$, and
\begin{align*}
\mu=&\frac{3}{\sqrt{2}} \frac{\|\te{T}\|_F}{g_{\min}} \frac{[(1+\sin\theta_0)^N-1]}{\sin\theta_0}\times
\frac{6g_{\max}(C_1+C_2)\sqrt{\frac{\alpha}{p}}}{g_{\min}- \sqrt{2}\kappa \psi_0 - 6g_{\max}(C_1+C_2)\sqrt{\frac{\alpha}{p}} \sin\theta_0 }<\frac{\gamma}{7\Gamma}.
\end{align*}
Then
\begin{align*}
\|\llbracket \te{G}_t; \ma{A}_t^{(1)},\dots,\ma{A}_t^{(N)}\rrbracket -\te{T}\|_F
\le \frac{7\mu \Gamma}{\gamma}\|\llbracket \te{G}_{t-1}; \ma{A}_{t-1}^{(1)},\dots,\ma{A}_{t-1}^{(N)}\rrbracket -\te{T}\|_F,\quad \mbox{w.h.p.}
\end{align*}
In other words, Algorithm~\ref{alg:tc} converges to the exact solution at a linear rate, {\it w.h.p}.
\end{theorem}

\begin{remark}
Let $\texttt{G}$ be a single node.
Then we are in fact solving LRTC problem with the multi-linear rank.
Let $I_1=\dots=I_N=d$, $\rank_n(\te{T})=(r,\dots,r)$.
Then $p_*d^N = O(Nr^{N-1}d\log d )$, i.e.,
we need at least $O(Nr^{N-1}d\log d)$ observations.
When $N=2$, we need $O(rd\log d)$ observations, which is theoretical optimal.
When $N=3$, we need $O(r^2 d \log d)$ observations.
Compared with existing multi-linear rank based methods, Algorithm~\ref{alg:tc} needs less observations, see Table~\ref{tab} below.
\end{remark}

\begin{table}[!ht]
\caption{ Bound for $|\Omega|$ for exact recovery -- the case for 3rd tensor of size $d\times d\times d$}
\label{tab}
\centering\vspace{0.1in}
\begin{tabular}{c||c}
\hline
Method & Bound for $|\Omega|$ \\
\hline
SNN~\cite{Proc:Mu_ICML14,Proc:Tomioka_NIPS11} & $O(rd^2)$ \\
Square Deal~\cite{Proc:Mu_ICML14} & $O(rd^2)$ \\
GoG~\cite{Article:Xia_FCSM19} & $O(r^{\frac72} d^{\frac32}\log^{\frac72}d+r^7 d \log^6 d)$ \\
TLALS (ours) & $O(r^2 d\log d)$ \\
\hline
\end{tabular}
\end{table}

\section{Numerical Experiment}\label{sec:numer}

In this section, we present several numerical examples to illustrate the performance of our method.

\vspace{0.1in}
\noindent{\bf Example 1.}\;
In this example, we let $\texttt{G}$ have a single node.
We compare our algorithm with two multi-linear rank based tensor completion methods, namely, geomCG~\cite{Article:Kressner_BNM14} and Tmac~\cite{xu2013parallel}.
\footnote{According to~\cite{Article:Sobral_PRL17}, Tmac ranked No.1 among 10 tensor completion methods over SBI data.}
The MATLAB codes for geomCG and Tmac are obtained from Github.
\footnote{https://github.com/andrewssobral/mctc4bmi}


We generate the tensor as
$\te{T} =\llbracket \te{G}; \ma{A}^{(1)},\dots,\ma{A}^{(N)}\rrbracket$,
where $\ma{A}^{(n)}\in\R^{I\times r_n}$,
$\te{G}\in\R^{r_1\times \dots\times r_N}$, and
their entries are {\em i.i.d.} from the standard normal distribution.
And each entry of $\te{T}$ is observed independently with probability $p$.
We perform the tests under the following settings:

\begin{enumerate}
\item[1] $I=(50,50,50)$, $R=(10,10,10)$, $p=0.1,0.2,0.3$;
\item $I=(50,50,50)$, $p=0.2$, $R=(r,r,r)$ for $r=5,10,20$;
\item $I=(50,50,50)$, $R=(r,r,r)$ for $r=5,10,\dots,25$, $p=0.05,0.10,\dots,0.4$.
\end{enumerate}

We use the residual
$\tau_t = \frac{\|\PO(\llbracket \te{G}_t; \ma{A}_t^{(1)},\dots,\ma{A}_t^{(N)}\rrbracket - \te{T})\|_F}{\|\PO(\te{T})\|_F}$
to measure the quality of the computed solution.

\begin{figure}[!hbt]
\vspace{0.1in}
\centering
\mbox{
\includegraphics[width=3in]{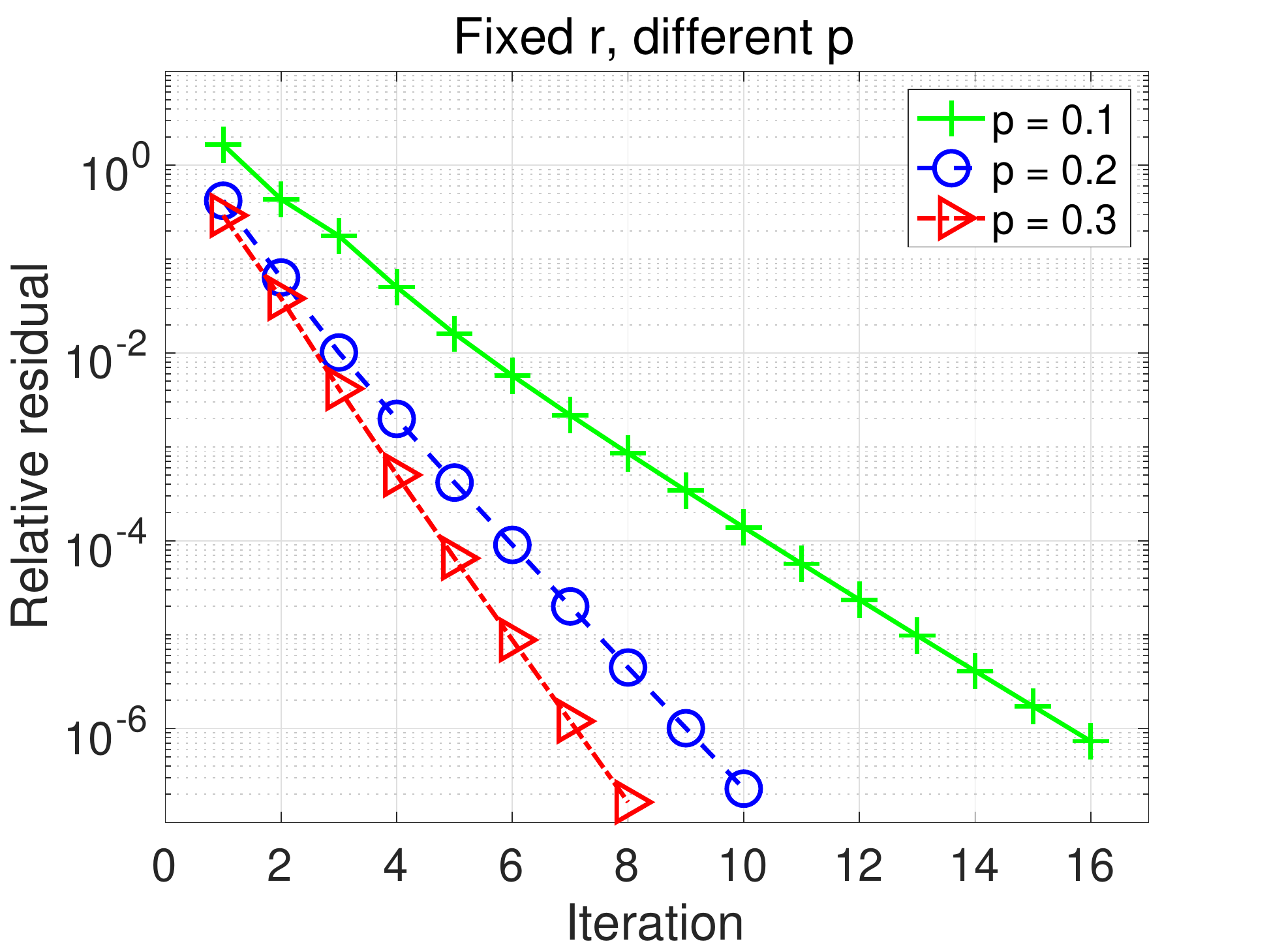}
\includegraphics[width=3in]{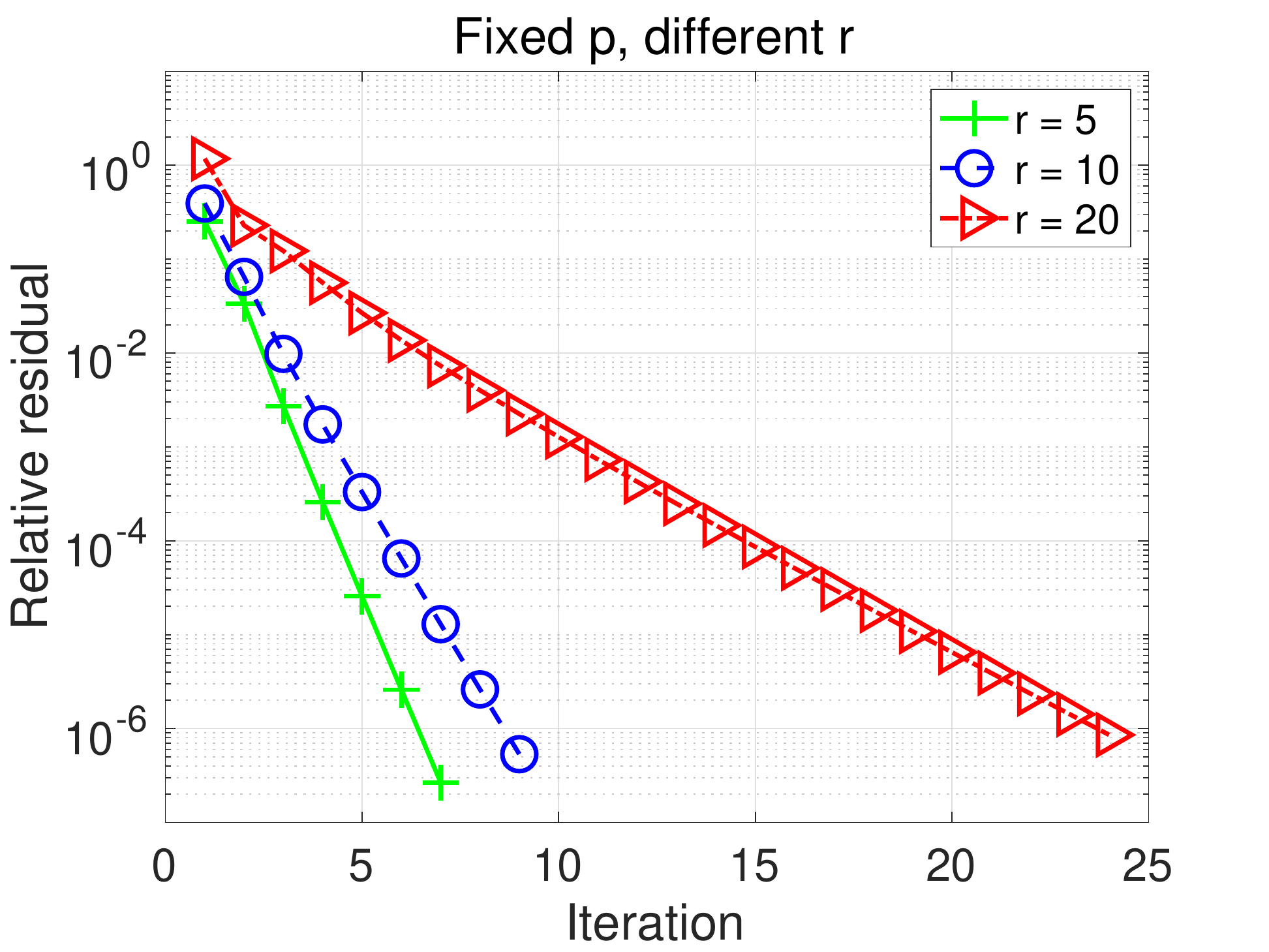}
}
\caption{Relative residual vs. iteration number, left: setting 1, right: setting2}
\label{fig1}
\end{figure}

In Figure~\ref{fig1}, we plot the results of our method under settings 1 and 2.
From the left (setting 1) and right (setting 2) figures,
we can see that in all cases, $\tau_t$ converges linearly;
the larger $p$ is, the larger the convergence rate is;
the smaller $r$ is, the larger the convergence rate is.

Under setting 3, we set the tolerance for all three methods to $10^{-4}$.
On the output of each method, if $\tau_t<10^{-2}$, we take it as a success.
For every pair of $(r,p)$, we perform all three methods 20 times.
The phase transitions for three methods are reported in Figure~\ref{fig2}.
We can see that to ensure a successful recovery,
our method permits a smaller $p$ and a larger $r$, compared with geomCG and Tmac.

\begin{figure}[!hbt]
\centering
\mbox{
\includegraphics[width=2.2in]{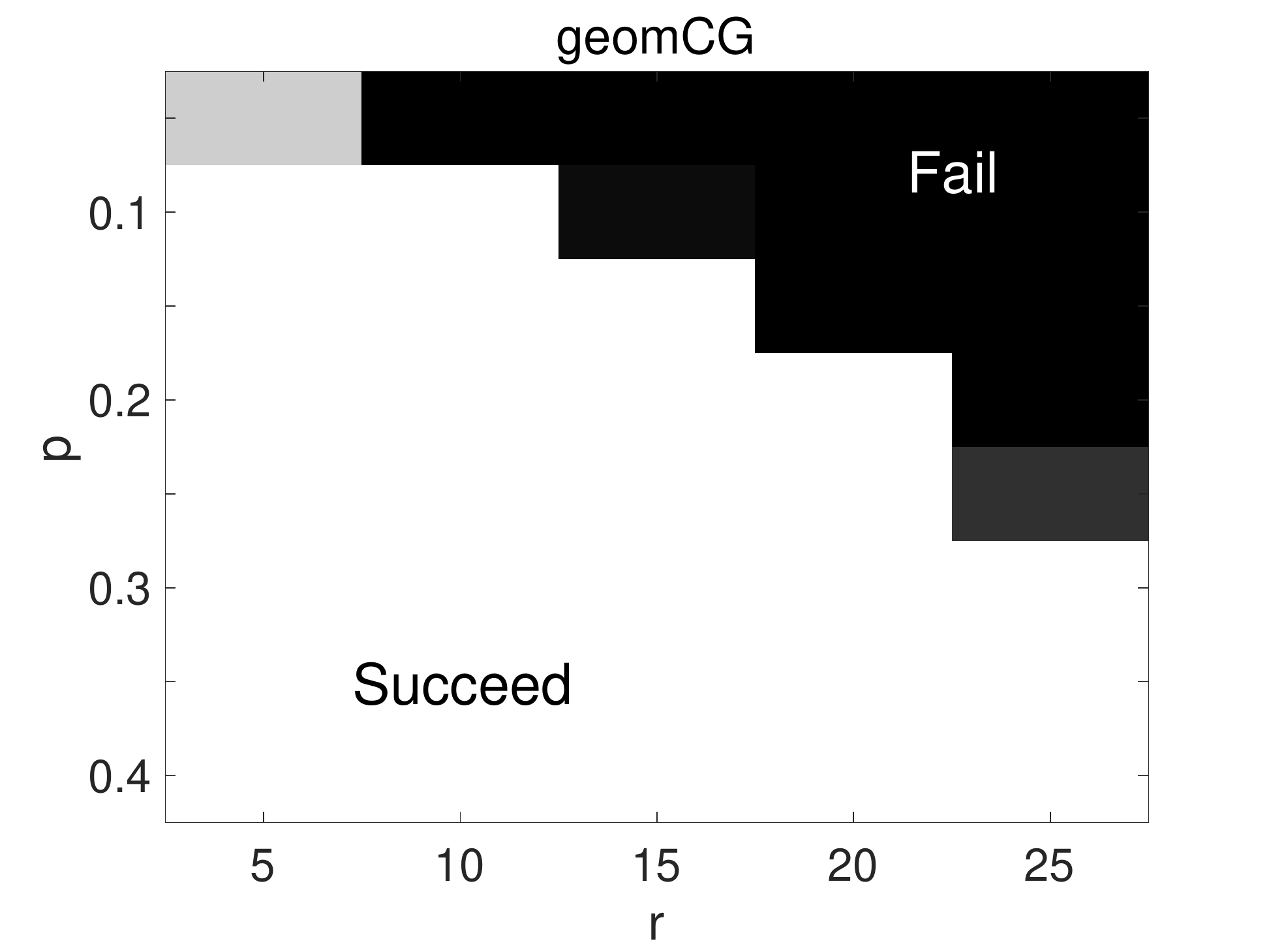}
\includegraphics[width=2.2in]{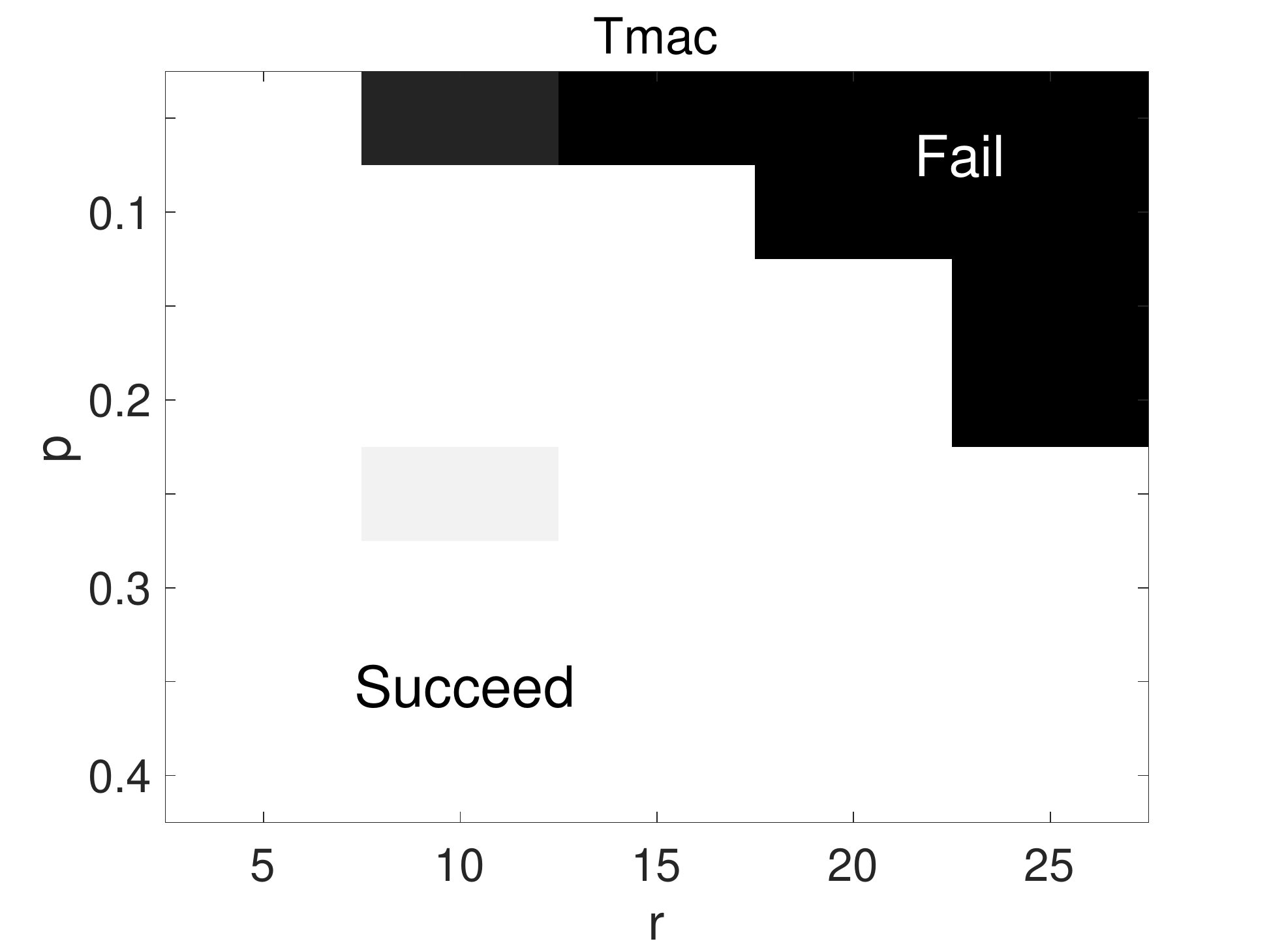}
\includegraphics[width=2.2in]{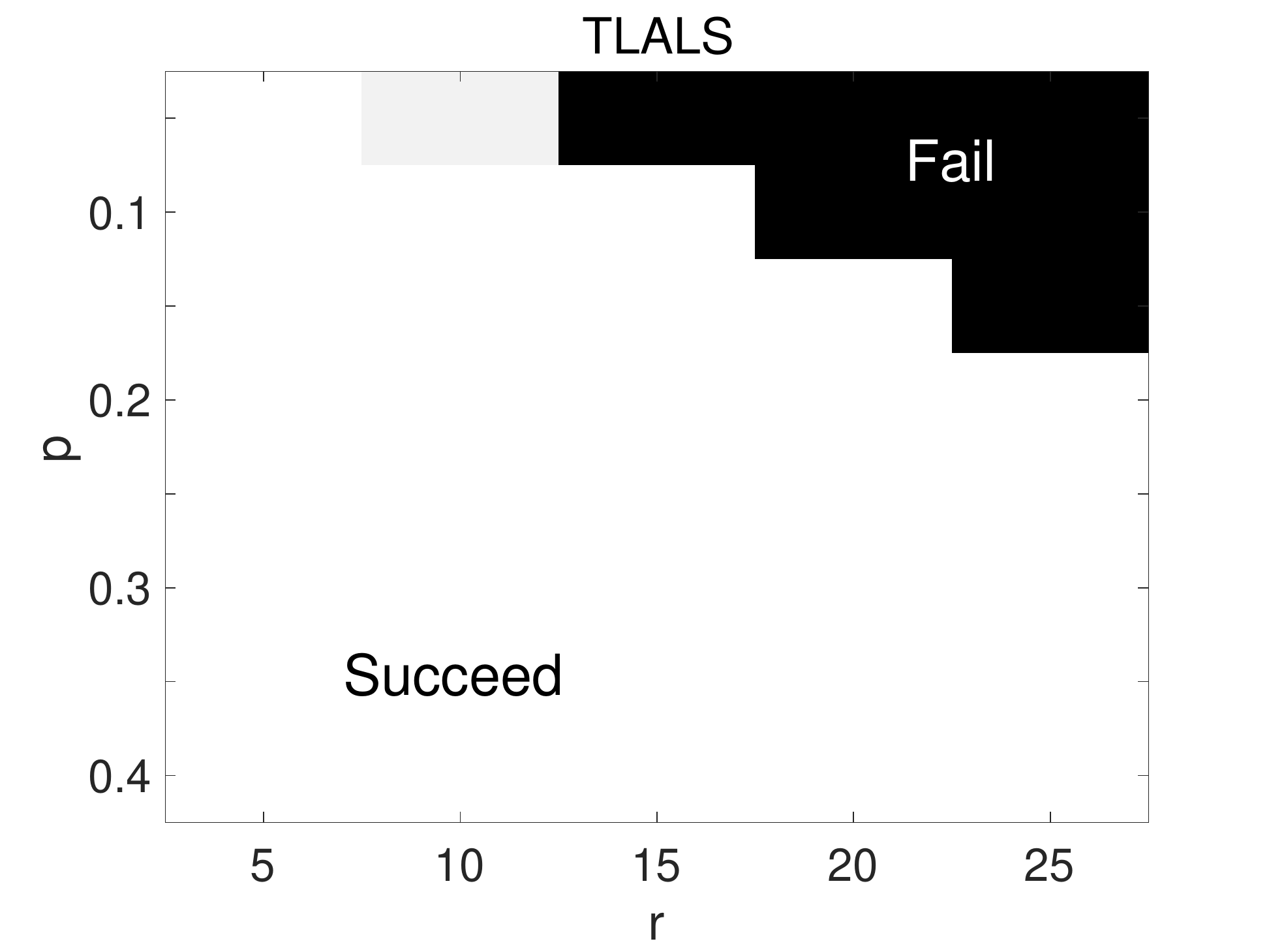}
}
\mbox{
GeomCG \hspace{1.6in} Tmac \hspace{1.7in} TLALS
}
\caption{Phase transition of successful recovery for three methods}
\label{fig2}
\end{figure}

\noindent{\bf Example 2.}\;
In this example, we show the effectiveness of the TN with a Tucker wrapper.
For the color image ``onion.png'' (available in MATLAB),
which is a third-order tensor of dimension $135\times198\times3$,
we reshape it into a fifth-order tensor of dimension $15\times9\times 18\times 11\times3$.
Each entry of the tensor is observed with a probability of $p=0.5$.
The original and observed figures are shown in Figure~\ref{fig4}.
We perform our method with three different tensor diagrams,
the recovered figures are shown in Figure~\ref{fig5}.
We can see that compared with simple Tucker,
Tucker with TT and TR improves the quality of the recovered figures.

\newpage

\begin{figure}[!hbt]
\centering
\includegraphics[width=2.1in]{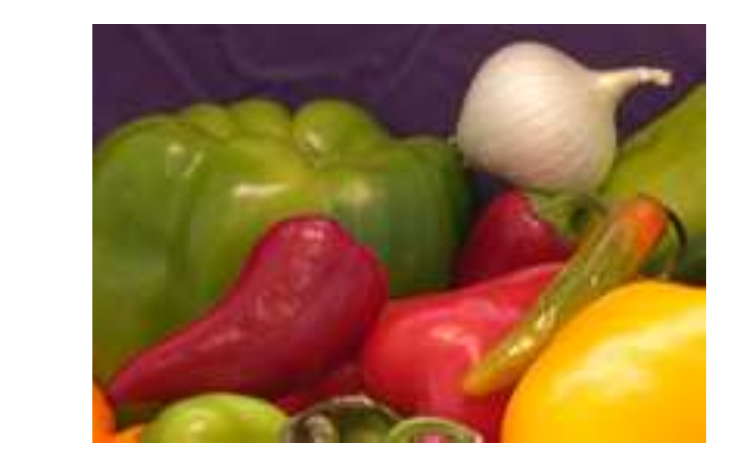}
\includegraphics[width=2.1in]{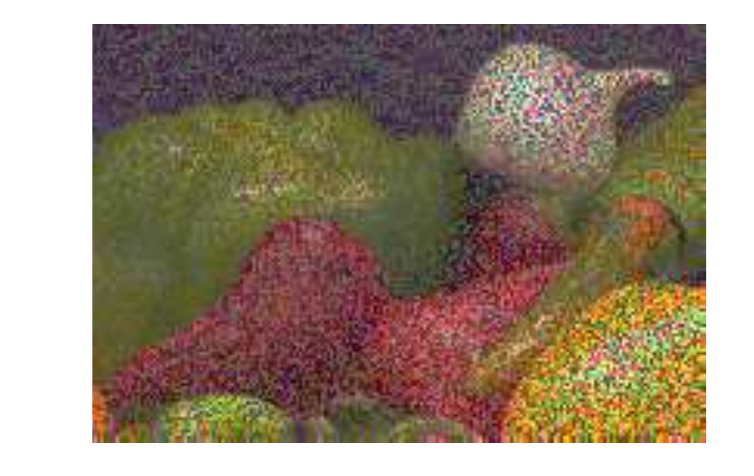}
\vspace{-0.1in}
\caption{Original and observed figures, left: original, right:observed}
\label{fig4}\vspace{0.2in}
\end{figure}

\begin{figure}[!ht]
\centering
\includegraphics[width=2in]{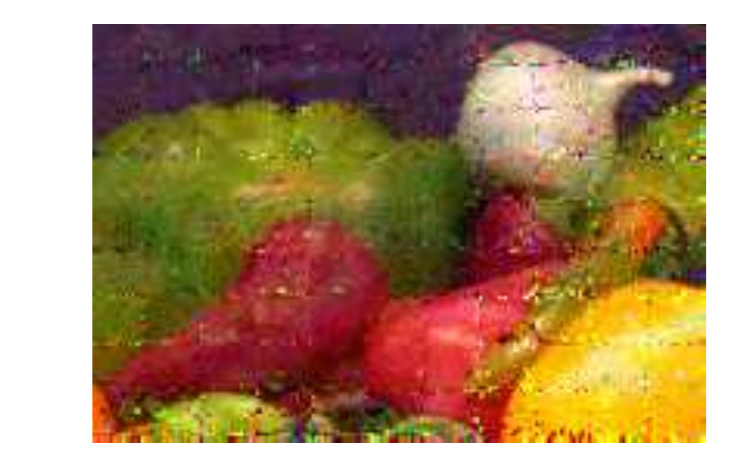}
\includegraphics[width=2in]{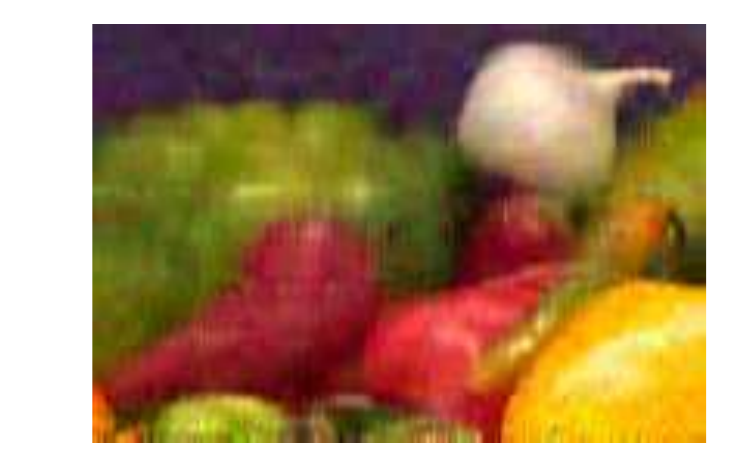}
\includegraphics[width=2in]{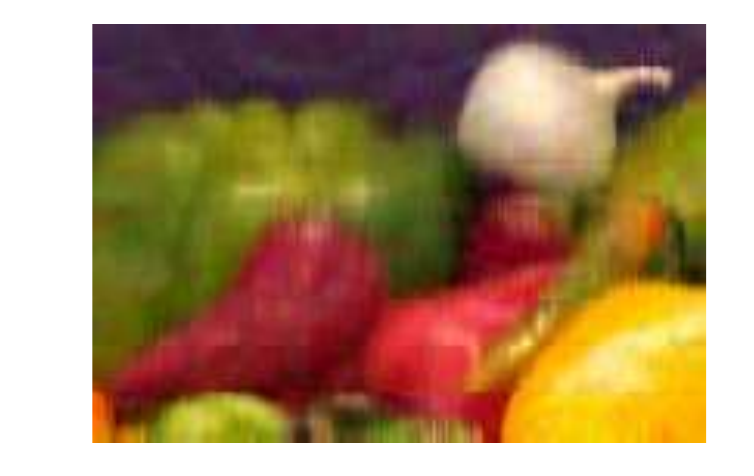}\\
{\sc psnr}$=20.75$ \hspace{1.2in} {\sc psnr}$=26.38$ \hspace{1.2in} {\sc psnr}$=26.83$\\
\vspace{0.1in}
\includegraphics[width=1.8in]{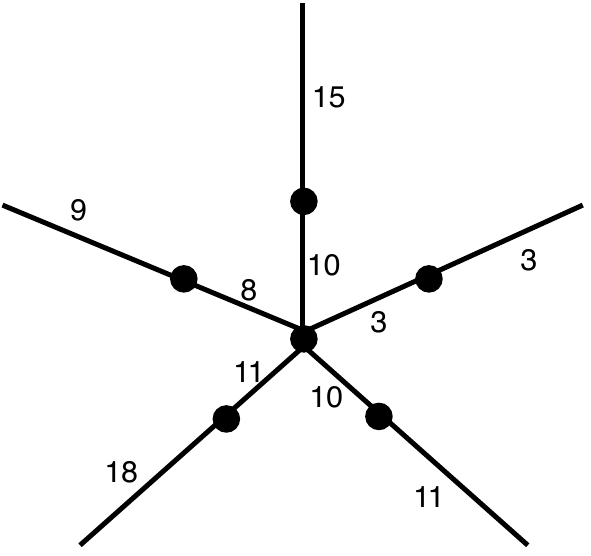}\hspace{0.1in}
\includegraphics[width=1.8in]{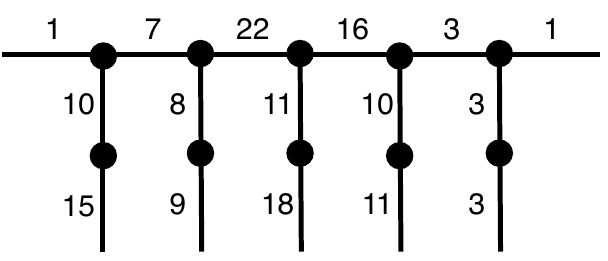}\hspace{0.1in}
\includegraphics[width=1.8in]{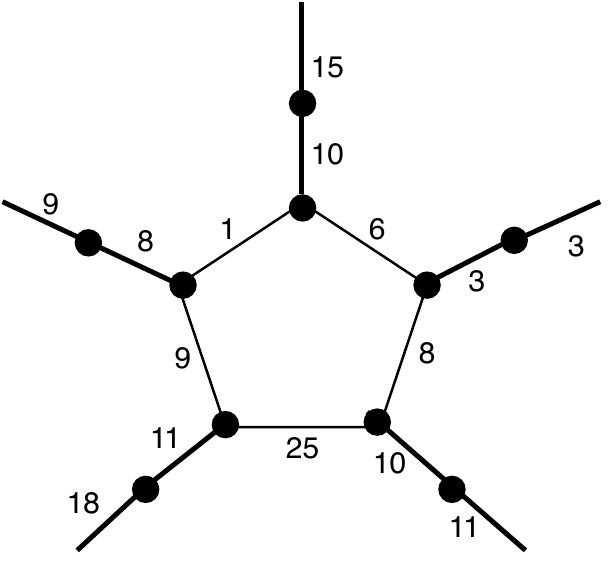}
\caption{Recovered figure, PSNR and the corresponding tensor diagram, from left to right: Tucker, Tucker + tensor train, Tucker + tensor ring}
\label{fig5}
\end{figure}

\section{Conclusion}\label{sec:conclusion}

This paper considers the LRTC problem via tensor networks with a Tucker wrapper.
The problem is formulated as a system of nonlinear equations,
and a two-level ALS method is proposed to solve it.
Under proper assumptions, it is shown that our method converges to the exact solution at a linear rate, {\it w.h.p}.
Also, to ensure convergence,
our method requires less number of observations compared with
existing multi-linear rank based tensor completion methods.
Numerical simulations show that the method is comparable with state-of-the-art algorithms.

The topology structure search problem remains still open: how to find the tensor diagram to achieve the best performance for real-world data?
More studies are needed to that end.

\clearpage

\bibliographystyle{plain}

\bibliography{standard}

\newpage\clearpage


\noindent{\bf\large Appendix}

\section{Preliminary Lemmas}
In this section, we give preliminary lemmas that will be frequently used in subsequent proofs. Firstly, the following lemma is fundamental for canonical angle, and can be easily verified via definition.\\

\begin{lemma}\label{lem:sin}
    Let $[\ma{U},\, \ma{U}_{\rm c}]$ and $[\ma{V},\, \ma{V}_{\rm c}]$ be two orthogonal matrices with $\ma{U}, \ma{V}\in\R^{n\times k}$. Then
    \begin{align*}
    \|\sin\Theta(\ma{U},\ma{V})\|
    &=\|\ma{U}_{\rm c}^{\T} \ma{V}\|=\|\ma{U}^{\T} \ma{V}_{\rm c}\|
    =\|(I-\ma{U} \ma{U}^{\T}) \ma{V}\|=\|(I - \ma{V} \ma{V}^{\T})\ma{U}\|
    =\|\ma{U} \ma{U}^{\T} - \ma{V} \ma{V}^{\T}\|.
    \end{align*}
\end{lemma}

The next lemma is the well-known Weyl theorem,  which gives the perturbation bound for singular values.
\begin{lemma}\cite[Corollary 5.1]{demmel1997applied}
    \label{lem:eig}
    For two matrices $\ma{A}$, $\wtd{\ma{A}}\in\mathbb{C}^{m\times n}$,
    it holds
    \[
    |\sigma_j(\wtd{\ma{A}})-\sigma_j(\ma{A})| \le \|\ma{A}-\wtd{\ma{A}}\|,\quad
    \forall j.
    \]
\end{lemma}

 \begin{lemma}\label{lem:prob}\cite[Lemmas 8,10]{Proc:Jain_COLT15} 
Let $\ma{A}\in\R^{m\times n}$ with $m\ge n$. Suppose $\Omega$ is obtained by sampling each entry of $\ma{A}$ with probability $p\in [\frac{1}{4m}, 0.5]$.
Then w.p. $\ge 1-1/m^{10+\log \alpha}$, it holds
\begin{align*}
\|\frac1p \PO(\ma{A})-\ma{A}\|\le \frac{6\sqrt{\alpha m}}{\sqrt{p}} \|\ma{A}\|_{\max}.
\end{align*}
\end{lemma}

 \begin{lemma}\label{lem:bernstein}\cite[Corollary 6.1.2]{Article:Tropp_FTML15}
 Let ${\bf S}_1,\dots, {\bf S}_n$ be independent random matrices with common dimension $d_1\times d_2$,
 and assume that each matrix has uniformly bounded deviation from its mean:
 \[
\|{\bf S_k} - \mathbb{E}({\bf S}_k)\|\le L, \quad \mbox{ for each } k=1,\dots, n.
 \]
 Let ${\bf Z}=\sum_{k=1}^n {\bf S}_k$, $v({\bf Z})$ denote the matrix cariance statistic of the sum:
 \begin{align*}
 v({\bf Z})
  =&\max\{\|\mathbb{E}[({\bf Z} -\mathbb{E}({\bf Z}))({\bf Z} -\mathbb{E}({\bf Z}))^{\T}]\|,
 \|\mathbb{E}[({\bf Z} -\mathbb{E}({\bf Z}))^{\T}({\bf Z} -\mathbb{E}({\bf Z}))]\| \}\\
=&\max\{\|\mathbb{E}[\sum_{k=1}^n({\bf S}_k-\mathbb{E}({\bf S}_k))({\bf S}_k-\mathbb{E}({\bf S}_k))^{\T}]\|,
\|\mathbb{E}[\sum_{k=1}^n({\bf S}_k-\mathbb{E}({\bf S}_k))^{\T}({\bf S}_k-\mathbb{E}({\bf S}_k))]\|\}.
 \end{align*}
 Then for all $t\ge 0$, it hols that
 \[
 \mathbb{P}\{\|{\bf Z}-\mathbb{E}({\bf Z})\|\ge t\}\le (d_1+d_2)\cdot \exp\Big(\frac{-t^2/2}{v({\bf Z})+Lt/3}\Big).
 \]
 \end{lemma}

\begin{lemma}\label{lem:sigmin}
Let $\ma{A}\in\R^{n\times r}$ have orthonormal columns and $\|\ma{A}\|_{2,\infty}\le\sqrt{\frac{\mu r}{n}}$.
Assume $p\ge 4p'$ with $p'=\frac{10}{3}(\log2n+5)\frac{\mu r}{n}$.
Then for any $\J$ that is uniformly  drawn from $\{1,\dots,n\}$ with probability $p$, it holds w.p. $\ge 0.99$ that
$\sigma^2_{\min}(\ma{A}_{(\J,:)})\ge p -\sqrt{pp'}$.
\end{lemma}

\begin{proof}
Let  $\{\delta_k\}_{1\le k\le n}$ be an independent family of {\sc Bernoulli}($p$) random variables,
and $\ma{A}^{\T}=[\ve{a}_1,\dots,\ve{a}_n]$.
Denote $\ma{W}_k = \delta_k \ve{a}_k \ve{a}_k^{\T}$,
$\ma{W}=\sum_{k=1}^n \ma{W}_k$.
Using $\ma{A}^{\T}\ma{A}=\ma{I}_r$ and $\|\ma{L}\|_{2,\infty}\le\sqrt{\frac{\mu r}{n}}$, by calculations, we have
\begin{align*}
& \mathbb{E}(\ma{W}) = \sum_k \mathbb{E}(\delta_k) \ve{a}_k \ve{a}_k^{\T} = p \ma{A}^{\T}\ma{A} = p \ma{I}_r,\\
&\mathbb{E}(\ma{W}_k)=  \mathbb{E}(\delta_k) \ve{a}_k\ve{a}_k^{\T} = p \ve{a}_k\ve{a}_k^{\T}, \hskip 1in
 \|\ma{W}_k-\mathbb{E}(\ma{W}_k)\|=|\delta_k-p| \|\ve{a}_k \ve{a}_k^{\T}\|< \frac{\mu r}{n},\\
&\|\mathbb{E}[\sum_k (\ma{W}_k-\mathbb{E}(\ma{W}_k))^2]\|
=\|\sum_k \mathbb{E}[(\delta_k-p)^2] (\ve{a}_k \ve{a}_k^{\T})^2\|
\le  \frac{p(1-p) \mu r}{n}.
\end{align*}
Then by Lemma~\ref{lem:bernstein}, we have
\begin{align*}
\mathbb{P}\{\|\ma{W}- p \ma{I}_r\|>t\} \le 2n \exp\Big(\frac{-t^2/2}{p(1-p)\mu r /n  + \mu r /n \times t/3}\Big).
\end{align*}
Let $t=\sqrt{p p'}$, then the right hand side of the above inequality satisfies
\begin{align*}
\textsc{rhs} &\le 2n \exp\Big(- \frac{1/2 pp' }{p \mu r /n + \mu r /n\times 2p/3 }\Big)
=\exp(-5) <0.01.
\end{align*}
Therefore, w.p. $\ge 0.99$, it holds
$\|\ma{W} - p \ma{I}_r\|\le \sqrt{pp'}$.
Using Lemma~\ref{lem:eig}, we have
\[
\sigma_r(\ma{W})\ge p - \sqrt{pp'}, \quad \mbox{w.p.} \ge 0.99,
\]
which completes the proof.
\end{proof}

\vspace{0.1in}

\begin{lemma}\label{lem:vecp}
Let $\ve{a}=[a_1,\dots,a_n]$,
$\ve{b}=[\delta_1a_1,\dots,\delta_n a_n]$ be two $n$-dimensional real vectors,
where $\{\delta_k\}_{1\le k\le n}$ is an independent family of {\sc Bernoulli}($p$) random variables.
Assume $p\ge \frac{256}{9}\frac{\|\ve{a}\|_{\max}^2}{\|\ve{a}\|^2}$.
Then  w.p. $\ge 0.996$, it holds that
$\frac{\sqrt{p}}{2} \|\ve{a}\|\le\|\ve{b}\|\le \frac{\sqrt{7p}}{2} \|\ve{a}\|$.
\end{lemma}

\vspace{0.1in}

\begin{proof}
Let ${x}_k = (\delta -p)a_k^2 $ for $k=1,\dots, n$, $z=\sum_{k=1}^n {x}_k$.
 Then
\begin{align*}
&\mathbb{E}({x}_k)= 0,\quad
|x_k|\le \|\ve{a}\|_{\max}^2,\quad
\mathbb{E}(z^2)=\sum_k\mathbb{E}(\delta -p)^2a_k^4
= p(1-p)\sum_k a_k^4 < p \|a\|_{\max}^2\|\ve{a}\|^2.
\end{align*}
By Lemma~\ref{lem:bernstein}, we have
\begin{align*}
\mathbb{P}\{|z|>t\} \le 2 \exp\Big(\frac{-t^2/2}{p \|a\|_{\max}^2\|\ve{a}\|^2  + \|\ve{a}\|_{\max}^2 \times t/3}\Big).
\end{align*}
Let $t=4\|\ve{a}\|_{\max}\|\ve{a}\| \sqrt{p}$, then the right hand side of the above inequality satisfies
\begin{align*}
\textsc{rhs} &\le 2\exp\Big(- \frac{8 p }{p  +   4  \|a\|_{\max}/\|\ve{a}\times \sqrt{p}/3  }\Big)
\le 2\exp\Big(- \frac{8  }{  1+   1/4   }\Big)
 <0.004.
\end{align*}
Therefore, w.p. $\ge 0.996$, it holds
$|z|\le 4\|\ve{a}\|_{\max}\|\ve{a}\| \sqrt{p}$.
Then it follows
\begin{align*}
\|\ve{b}\|^2 &\le p\|\ve{a}\|^2 + 4\|\ve{a}\|_{\max}\|\ve{a}\| \sqrt{p}
\le p\|\ve{a}\|^2 + \frac34 p \|\ve{a}\|^2  = \frac74 p\|\ve{a}\|^2,\\
\|\ve{b}\|^2&\ge p\|\ve{a}\|^2 - 4\|\ve{a}\|_{\max}\|\ve{a}\| \sqrt{p}
\ge p\|\ve{a}\|^2 - \frac34 p \|\ve{a}\|^2 = \frac14 p \|\ve{a}\|^2,
\end{align*}
completing the proof.
\end{proof}

\section{Proofs of the Main results}

\noindent{\bf Proof of Lemma~4.1}\\

\noindent{\bf Lemma 4.1.}\;
Denote
$J_n=\prod_{k\ne n} I_k$, $J_{\min}=\min_k J_k$,
$g_{\max}=\max_k \|\te{T}_{(k)}\|$,
and $\sin\theta_t=\max_{1\le k\le N} \|\sin\Theta(\ma{A}_t^{(k)},\ma{A}_*^{(k)})\|$. Let
\begin{align*}
\te{L}(\ma{X})&=\llbracket \te{G}_{t-1}; \ma{A}_t^{(1)},\dots,\ma{A}_t^{(n-1)},\ma{X},\ma{A}_{t-1}^{(n+1)},\dots,\ma{A}_{t-1}^{(N)}\rrbracket,\\
\ma{X}_{\opt}&=\argmin \| \te{L}_t(\ma{X}) - \te{T}\|_F,\\
\wtd{\ma{X}}_{\opt}&=\argmin \|\PO(\te{L}_t(\ma{X})) - \PO(\te{T})\|_F.
\end{align*}
Assume {\bf A1-A3} and {\bf A5}, $p\in [4p_*,0.5]$
with $p_*= \frac{10}{3} \big(\log(2\prod_{k=1}^N I_k) +5\big) \max_n\prod_{k\ne n} {\frac{\mu_{k} r_k}{I_k}}$
and $\sigma_{\min}([\te{G}_t]_{(n)})\ge g_{\min}$ for all $n$ and $t$.
Also assume
\begin{equation*}
\|\sin\Theta(\ma{M}_{t,n} [\te{G}_{t-1}]_{(n)}^{\T}, \ma{M}_{*,n} [\te{G}_*]_{(n)}^{\T})\|\le C\sin\theta_{t-1},
\end{equation*}
where $\ma{M}_{t,n}$, $\ma{M}_{*,n}$ are defined in \eqref{mm}, $C>0$ is a constant.
Then w.p. $\ge 1 - 2/J_{\min}^{10+\log\alpha}$, it holds that
\begin{equation*}
\|\wtd{\ma{X}}_{\opt} - \ma{X}_{\opt}\|\le \frac{6g_{\max}(C_1+C_2)}{g_{\min}}\sqrt{\frac{\alpha}{p}} \sin\theta_{t-1},
\end{equation*}
where
$C_1=C \frac{g_{\max}}{g_{\min}} \max_n\{ \sqrt{\frac{7 \mu_n r_n }{I_n}} \prod_{k\ne n} \sqrt{\mu_k r_k}\}$, and
$C_2= C \max_n\sqrt{\frac{\mu_n r_n J_n}{I_n}} $.

\begin{proof}
Recall \eqref{mm}. We have
\begin{align*}
\ma{X}_{\opt} = \ma{A}_*^{(n)} [\te{G}_*]_{(n)} \ma{M}_{*,n}^{\T} \ma{M}_{t,n} [\te{G}_{t-1}]_{(n)}^{\dagger}.
\end{align*}
and $\wtd{\ma{X}}_{\opt}$ is the solution to
\begin{align}\label{txopt}
\|\ma{P}_{\Omega}(\ma{X} [\te{G}_{t-1}]_{(n)} \ma{M}_{t,n}^{\T}) - \ma{P}_{\Omega}(\ma{A}_*^{(n)} [\te{G}_*]_{(n)} \ma{M}_{*,n}^{\T})\|=\min,
\end{align}
where $\ma{P}_{\Omega}$ picks the entries of the unfolding matrix corresponding to the observed entries in $\Omega$.
Set $\ma{X} = \ma{X}_{\opt} + \ma{Y}$ in \eqref{txopt}.
Since the smallest singular value of the linear operator $\ma{P}_{\Omega}(\ \cdot \ [\te{G}_{t-1}]_{(n)} \ma{M}_{t,n}^{\T})$
is positive (can be easily shown via Lemma~\ref{lem:sigmin}) with high probability, we know that
$\ma{Y} = \wtd{\ma{X}}_{\opt} - {\ma{X}}_{\opt}$ is the solution to
\begin{align}\label{deltax}
\|\ma{P}_{\Omega}(\ma{Y} [\te{G}_{t-1}]_{(n)} \ma{M}_{t,n}^{\T})
-\ma{P}_{\Omega} (\ma{R})\|=\min,
\end{align}
where
\[
\ma{R}=\ma{A}_*^{(n)} [\te{G}_*]_{(n)} \ma{M}_{*,n}^{\T} (\ma{I} - \ma{M}_{t,n} [\te{G}_{t-1}]_{(n)}^{\dagger} [\te{G}_{t-1}]_{(n)} \ma{M}_{t,n}^{\T}).
\]

On one hand, by calculations, we have
\begin{align}
\|\ma{R}\|_{\max}
&\stackrel{(a)}{\le} \sqrt{\frac{\mu_n r_n}{I_n}} g_{\max} \|\sin\Theta(\ma{M}_{t,n} [\te{G}_{t-1}]_{(n)}^{\T}, \ma{M}_{*,n} [\te{G}_*]_{(n)}^{\T})\|
\stackrel{(b)}{\le } \sqrt{\frac{\mu_n r_n}{I_n}} g_{\max} C \sin\theta_{t-1},
\label{rmax}
\end{align}
where (a) uses {\bf A3} and Lemma~\ref{lem:sin},
(b) uses \eqref{mtheta}.

On ther other hand, rewrite \eqref{deltax} as $I_n$ independent least square problems:
\begin{align}
\|\ve{e}_i^{\T} \ma{Y} [\te{G}_{t-1}]_{(n)} \ma{M}_{t,n}^{\T} \ma{P}_i^{\T} - \ve{e}_i^{\T}\ma{R}\ma{P}_i^{\T}\|=\min,
\quad \mbox{for } i=1,\dots,I_n,
\end{align}
where $\ma{P_i}^{\T}$ picks the observed entries on the $i$th row of the unfolding matrix.

By calculations,  we have w.p. $\ge 0.99$ that
\begin{align}
\|\ve{e}_i^{\T} \ma{Y}\|
&= \|\ve{e}_i^{\T}\ma{R}\ma{P}_i^{\T}([\te{G}_{t-1}]_{(n)} \ma{M}_{t,n}^{\T} \ma{P}_i^{\T})^{\dagger}\|
\le \frac{1}{g_{\min}} \|\ve{e}_i^{\T}\ma{R}\ma{P}_i^{\T}([\wht{\te{G}}_{t-1}]_{(n)} \ma{M}_{t,n}^{\T} \ma{P}_i^{\T})^{\dagger}\|\notag\\
&\le \frac{1}{g_{\min}} \|\ve{e}_i^{\T}\ma{R}\ma{P}_i^{\T}\| \|([\wht{\te{G}}_{t-1}]_{(n)} \ma{M}_{t,n}^{\T} \ma{P}_i^{\T})^{\dagger}\|
\stackrel{(d)}{\le} \frac{1 }{g_{\min} \sqrt{p-\sqrt{pp'}}} \|\ve{e}_i^{\T}\ma{R}\ma{P}_i^{\T}\|\notag\\
&\stackrel{(e)}{\le} \frac{1 }{g_{\min} \sqrt{p/2}} \frac{\sqrt{7p}}{2}\sqrt{\frac{\mu_n r_n}{I_n}}  g_{\max}
\|\sin\Theta(\ma{M}_{t,n} [\te{G}_{t-1}]_{(n)}^{\T}, \ma{M}_{*,n} [\te{G}_*]_{(n)}^{\T})\|\notag\\
&\stackrel{(f)}{\le} \frac{g_{\max} }{g_{\min} } \sqrt{\frac{7\mu_n r_n}{I_n} } C \sin\theta_{t-1}.\label{edx}
\end{align}
where the row vectors of $[\wht{\te{G}}_{t-1}]_{(n)}$ and $[\wht{\te{G}}_*]_{(n)}$ form orthonormal bases for
the subspaces spanned by the row vectors of $[{\te{G}_{t-1}}]_{(n)}$ and $[{\te{G}}_*]_{(n)}$, respectively,
 (d) uses Lemma~\ref{lem:sigmin},
 $\|\ma{M}_{t,n}[\wht{\te{G}}_{t-1}]_{(n)}^{\T}\|_{2,\infty}\le \prod_{k\ne n}\sqrt{\frac{\mu_k r_k}{I_k}}$,
 $p'=\frac{10}{3}(\log2J_n+5)\prod_{k\ne n}\frac{\mu_k r_k}{I_k}<p_*$,
(e) uses  {\bf A3}, Lemma~\ref{lem:sin}, Lemma~\ref{lem:vecp},
(f) uses \eqref{mtheta}.

By Lemma~\ref{lem:prob}, we have  w.p. $\ge 1- 1/J_n^{10+\log\alpha}$ that
\begin{align}\label{omeef}
\frac1p \ma{P}_{\Omega}(\ma{Y} [\te{G}_{t-1}]_{(n)} \ma{M}_{t,n}^{\T})
= \ma{Y} [\te{G}_{t-1}]_{(n)} \ma{M}_{t,n}^{\T} + \ma{E},\qquad
\frac1p \ma{P}_{\Omega}(\ma{R})
=\ma{R}+\ma{F},
\end{align}
where
\begin{align}
\|\ma{E}\| &\le \frac{6\sqrt{\alpha J_n}}{\sqrt{p}} \|\ma{Y} [\te{G}_{t-1}]_{(n)} \ma{M}_{t,n}^{\T}\|_{\max}
\le \frac{6\sqrt{\alpha J_n}}{\sqrt{p}} g_{\max}   \|\ma{M}_{t,n}\|_{2,\infty} \|\ma{Y}\|_{2,\infty}\notag\\
&\stackrel{(g)}{\le} \frac{6\sqrt{\alpha J_n}}{\sqrt{p}} g_{\max} \prod_{k\ne n} \sqrt{\frac{\mu_k r_k}{I_k} } \frac{g_{\max} }{g_{\min} } \sqrt{\frac{2\mu_n r_n}{I_n p} } C \sin\theta_{t-1}
\le 6g_{\max}C_1 \sqrt{\frac{\alpha}{p}}\sin\theta_{t-1},\label{e}\\
\|\ma{F}\| &\le \frac{6\sqrt{\alpha J_n}}{\sqrt{p}} \|\ma{R}\|_{\max}
\stackrel{(h)}{\le} \frac{6\sqrt{\alpha J_n}}{\sqrt{p}} \sqrt{\frac{\mu_n r_n}{I_n}} g_{\max} C \sin\theta_{t-1}
\le 6g_{\max}C_2 \sqrt{\frac{\alpha}{p}}\sin\theta_{t-1}.\label{f}
\end{align}
Here (g) uses \eqref{edx} and {\bf A3},
(h) uses \eqref{rmax}.

Using \eqref{deltax} and \eqref{omeef}, we get
\begin{align*}
\wtd{\ma{X}}_{\opt} - \ma{X}_{\opt} = \big(\ma{R} + \ma{F}-\ma{E}\big)
 \ma{M}_{t,n} [\te{G}_{t-1}]_{(n)}^{\dagger}
 =(\ma{F}-\ma{E}) \ma{M}_{t,n} [\te{G}_{t-1}]_{(n)}^{\dagger}.
\end{align*}
Combining it with \eqref{e} and \eqref{f}, we get
\begin{align}
\|\wtd{\ma{X}}_{\opt} - \ma{X}_{\opt}\|
\le (\|\ma{F}\|+\|\ma{E}\|)\|[\te{G}_{t-1}]_{(n)}^{\dagger}\|
\le \frac{6g_{\max}(C_1+C_2)}{g_{\min}}\sqrt{\frac{\alpha}{p}}\sin\theta_{t-1}.
\end{align}

This completes the proof.
\end{proof}

\vspace{0.2in}

\noindent{\bf Proof of Lemma~4.2}\\

\noindent{\bf Lemma 4.2.}\;
Let $\phi_t=\|\llbracket \te{X}_{\opt}; \ma{A}_t^{(1)},\dots,\ma{A}_t^{(N)}\rrbracket -\te{T}\|_F$,
$\psi_t=\|\llbracket \wtd{\te{X}}_{\opt}; \ma{A}_t^{(1)},\dots,\ma{A}_t^{(N)}\rrbracket -\te{T}\|_F$,
where
\begin{align*}
\te{X}_{\opt} &=\argmin\|\llbracket \te{X}; \ma{A}_t^{(1)},\dots,\ma{A}_t^{(N)}\rrbracket -\te{T}\|_F,\\
\wtd{\te{X}}_{\opt} &=\argmin\|\PO(\llbracket {\te{X}}; \ma{A}_t^{(1)},\dots,\ma{A}_t^{(N)}\rrbracket -\te{T})\|_F.
\end{align*}
Assume {\bf A1}, {\bf A3}, $p\in [4p_*,0.5]$
with $p_*$ being the same as in Lemma~\ref{lem:xxdif}.
Then w.p. $\ge 0.99$, it holds
\begin{align*}
\phi_t\le \psi_t \le 3/\sqrt{2}\ \phi_t.
\end{align*}

\begin{proof}
Let $\ma{L}=\ma{A}_t^{(N)}\otimes \dots \otimes \ma{A}_t^{(1)}$.
Then $\ma{L}$ has orthonormal columns since $\ma{A}_t^{(n)}$'s all have orthonormal columns.
Using $\|\ma{A}_t^{(n)}\|_{2,\infty} \le \sqrt{\frac{\mu_n r_n}{I_n}}$ (by {\bf A3}), we have $\|\ma{L}\|_{2,\infty}\le \prod_{k=1}^N\sqrt{\frac{\mu_k r_k}{I_k}}$.

Let
\begin{align*}
\te{X}_{\opt} &=\argmin\|\llbracket \te{X}; \ma{A}_t^{(1)},\dots,\ma{A}_t^{(N)}\rrbracket -\te{T}\|_F,\\
\wtd{\te{X}}_{\opt} &=\argmin\|\PO(\llbracket {\te{X}}; \ma{A}_t^{(1)},\dots,\ma{A}_t^{(N)}\rrbracket -\te{T})\|_F.
\end{align*}
Then $\te{X}_{\opt}$ and $\wtd{\te{X}}_{\opt}$ can also be given by
\begin{align*}
\vec(\te{X}_{\opt})=\argmin\|\ma{L}\vec(\te{X}) - \vec(\te{T})\|,\quad
\vec(\wtd{\te{X}}_{\opt})=\argmin\| \ma{P}_{\Omega}(\ma{L}\vec(\te{X}) -  \vec(\te{T}))\|.
\end{align*}
Then by Lemma~\ref{lem:sigmin}, w.p. $\ge 0.99$, it holds that
\begin{align}\label{pol}
\sigma_{\min}(\ma{P}_{\Omega}\ma{L})\ge \sqrt{p - \sqrt{pp_*}}\ge \sqrt{\frac{p}{2}}.
\end{align}

Rewrite $\vec(\te{T})= \vec(\te{T}_1) +\vec(\te{T}_2)$, where $\vec(\te{T}_1)\in\mathcal{R}(\ma{L})$ and
 $\vec(\te{T}_2)\in\mathcal{R}(\ma{L})^{\bot}$.
 Then it follows that
 \begin{align}
 \vec(\te{X}_{\opt})&=\ma{L}^{\dagger} \vec(\te{T}_1),\notag\\
\vec(\wtd{\te{X}}_{\opt})
&=(\ma{P}_{\Omega}\ma{L})^{\dagger}  \ma{P}_{\Omega}\vec(\te{T}_1) + (\ma{P}_{\Omega}\ma{L})^{\dagger}  \ma{P}_{\Omega}\vec(\te{T}_2)\notag \\
&\stackrel{(a)}{=}\vec(\te{X}_{\opt}) + (\ma{P}_{\Omega}\ma{L})^{\dagger}  \ma{P}_{\Omega}\vec(\te{T}_2),\label{vecxt}
\end{align}
where (a) uses the fact that $\mathcal{R}(\ma{P}_{\Omega}\vec(\te{T}_1))\subset \mathcal{R}(\ma{P}_{\Omega}\ma{L})$.

By calculations, we get
\begin{align*}
 \psi_t&=\|\llbracket \wtd{\te{X}}_{\opt}; \ma{A}_t^{(1)},\dots,\ma{A}_t^{(N)}\rrbracket -\te{T}\|_F\\
&=\|\ma{L}\vec(\wtd{\te{X}}_{\opt}) -\vec(\te{T})\|\\
&= \|\ma{L}\vec({\te{X}}_{\opt}) -\vec(\te{T}_1) + \ma{L} (\ma{P}_{\Omega}\ma{L})^{\dagger}  \ma{P}_{\Omega}\vec(\te{T}_2) -\vec(\te{T}_2)\|\\
&=\|\ma{L} (\ma{P}_{\Omega}\ma{L})^{\dagger}  \ma{P}_{\Omega}\vec(\te{T}_2) -\vec(\te{T}_2)\|\\
&\stackrel{(b)}{=} \sqrt{ \|\ma{L} (\ma{P}_{\Omega}\ma{L})^{\dagger}  \ma{P}_{\Omega}\vec(\te{T}_2)\|^2 +\|\vec(\te{T}_2)\|^2},
\end{align*}
where (b) uses $\ma{L} (\ma{P}_{\Omega}\ma{L})^{\dagger}  \ma{P}_{\Omega}\vec(\te{T}_2)\in\mathscr{R}(\ma{L})$,
and $\vec(\te{T}_2)\in\mathcal{R}(\ma{L})^{\bot}$.
The it follows that
\begin{align*}
\psi_t&\ge \|\vec(\te{T}_2)\|=\phi_t,\\
\psi_t&\stackrel{(c)}{\le} \sqrt{\frac{7p/4}{\sigma^2_{\min}(\ma{P}_{\Omega}\ma{L})}+1} \|\vec(\te{T}_2)\|
\stackrel{(d)}{\le} \frac3{\sqrt{2}}\phi_t,
\end{align*}
where (c) uses Lemma~\ref{lem:vecp},
(d) uses \eqref{pol}.
This completes the proof.
\end{proof}

\vspace{0.2in}

\noindent{\bf Proof of Theorem~4.3}.\\

To show Theorem~\ref{thm:main}, we also need the following two lemmas.

\begin{lemma}\label{lem:tt}
Follow the notations and assumptions in Lemmas~\ref{lem:xxdif} and~\ref{lem:xxcore}, and denote $\kappa=\max_k\kappa([\te{G}_*]_{(k)})$.
If
\[
g_{\min} > \sqrt{2}\kappa \psi_{t-1} -  6g_{\max}(C_1+C_2)\sqrt{\frac{\alpha}{p}} (\sin\theta_{t-1}+1),
\]
then w.p. $\ge 1 - 2/J_{\min}^{10+\log\alpha}$, it holds that
\[
\sin\theta_t\le \frac{6g_{\max}(C_1+C_2)\sqrt{\frac{\alpha}{p}} \sin\theta_{t-1}}{g_{\min}- \sqrt{2}\kappa \psi_{t-1} - 6g_{\max}(C_1+C_2)\sqrt{\frac{\alpha}{p}} \sin\theta_{t-1} }<\sin\theta_{t-1} .
\]
\end{lemma}

\begin{proof}
Denote $\te{S}=\llbracket \te{G}_{t-1}; \ma{A}_t^{(1)},\dots,\ma{A}_t^{(n-1)},\ma{I},\ma{A}_{t-1}^{(n+1)},\dots,\ma{A}_{t-1}^{(N)}\rrbracket$, and recall \eqref{mm}.
Then $\ma{X}_{\opt}=\argmin \|\ma{X}\te{S}_{(n)} -\te{T}_{(n)}\|_F = \te{T}_{(n)}\te{S}_n^{\dagger}$.
By calculations, we have
\begin{align}
\sigma_{\min}({\ma{X}}_{\opt})
&=\sigma_{\min}(\te{T}_{(n)}\te{S}_{(n)}^{\dagger})
=\sigma_{\min}\big(\ma{A}^{(n)}\te{G}_{(n)}\ma{M}_n^{\T} \ma{M}_{t,n}[\te{G}_{t-1}]_{(n)}^{\dagger} (\ma{A}_{t-1}^{(n)})^{\T}\big)\notag\\
&\stackrel{(a)}{\ge} 1 - \big\|\ma{A}^{(n)}\te{G}_{(n)}\ma{M}_n^{\T} (\ma{M}_{t,n}[\te{G}_{t-1}]_{(n)}^{\dagger} (\ma{A}_{t-1}^{(n)})^{\T}
- \ma{M}_n\te{G}_{(n)}^{\dagger}(\ma{A}^{(n)})^{\T} )\big\|\notag\\
&\stackrel{(b)}{\ge} 1 - \|\te{G}_{(n)}\|\big\|\ma{M}_{t,n}[\te{G}_{t-1}]_{(n)}^{\dagger} (\ma{A}_{t-1}^{(n)})^{\T}
- \ma{M}_n\te{G}_{(n)}^{\dagger}(\ma{A}^{(n)})^{\T}\big\|\notag\\
&\stackrel{(c)}{=} 1 - \|\te{G}_{(n)}\|\big\| ( \ma{A}_{t-1}^{(n)} [\te{G}_{t-1}]_{(n)} \ma{M}_{t,n}^{\T})^{\dagger}
-  (\ma{A}^{(n)} \te{G}_{(n)}\ma{M}_n^{\T})^{\dagger}\big\|\notag\\
&\stackrel{(d)}{\ge} 1- \|\te{G}_{(n)}\| \sqrt{2} \|[\te{G}_{t-1}]_{(n)}^{\dagger}\| \|\te{G}_{(n)}^{\dagger}\|
\|\ma{A}_{t-1}^{(n)} [\te{G}_{t-1}]_{(n)} \ma{M}_{t,n}^{\T} - \ma{A}^{(n)}\te{G}_{(n)}\ma{M}_n^{\T}\|\notag\\
&\stackrel{(e)} = 1 - \|\te{G}_{(n)}\| \sqrt{2}  \|[\te{G}_{t-1}]_{(n)}^{\dagger}\| \|\te{G}_{(n)}^{\dagger}\| \psi_{t-1}\notag\\
&\stackrel{(f)}{\ge} 1- \frac{\sqrt{2}\kappa_n}{g_{\min}} \psi_{t-1}, \label{sigminx2}
\end{align}
where (a) uses $\|\ma{A}^{(n)}\te{G}_{(n)}\ma{M}_n^{\T} (\ma{M}_n\te{G}_{(n)}^{\dagger}(\ma{A}^{(n)})^{\T})\|=1$,
(b) uses $\ma{A}^{(n)}$ has orthonormal columns,
(c) uses $\ma{A}^{(n)}$, $\ma{A}_{t-1}^{(n)}$, $\ma{M}_n$ and $\ma{M}_{t,n}$ all have orthonormal columns,
(d) uses~\cite[Theorem 4.1]{Article:Wedin_BNM73},
(e) uses $\ma{A}_{t-1}^{(n)} [\te{G}_{t-1}]_{(n)} \ma{M}_{t,n}^{\T} - \ma{A}^{(n)}\te{G}_{(n)}\ma{M}_n^{\T}
=[[\llbracket \te{G}_{t-1}; \ma{A}_t^{(1)},\dots,\ma{A}_{t-1}^{(N)}\rrbracket -\te{T}]_{(n)}$ and the definition of $\psi_t$,
(f) uses $\kappa(\te{G}_{(n)})\le \kappa_n$.

Then it follows that
\begin{align}
\sigma_{\min}(\wtd{\ma{X}}_{\opt})
&\stackrel{(g)}{\ge} \sigma_{\min}({\ma{X}}_{\opt}) - \|\wtd{\ma{X}}_{\opt} - \ma{X}_{\opt}\| \notag\\
&\stackrel{(h)}{\ge} 1- \frac{\sqrt{2}\kappa_n}{g_{\min}} \psi_{t-1}
-  \frac{6g_{\max}(C_1+C_2)}{g_{\min}}\sqrt{\frac{\alpha}{p}} \sin\theta_{t-1},
\quad \mbox{w.p. $\ge 1 - 2/J_{\min}^{10+\log\alpha}$,} \label{wtdx}
\end{align}
where (g) uses Lemma~\ref{lem:eig},
(h) uses \eqref{sigminx2} and Lemma~\ref{lem:xxdif}.

Let $[\ma{A}_*^{(n)}, \ma{A}_c^{(n)}]$ be an orthogonal matrix,
 the QR decomposition of $\wtd{\ma{X}}_{\opt}$ be $\wtd{\ma{X}}_{\opt} = \ma{Q}\ma{R}$, where $\ma{Q}$ has orthonormal columns, $\ma{R}$ is nonsingular.
Then we have
\begin{align*}
\|\sin\Theta(\ma{A}_t^{(n)},\ma{A}^{(n)})\|
&\stackrel{(i)}{=}\|(\ma{A}_c^{(n)})^{\T} \wtd{\ma{X}}_{\opt}\ma{R}^{-1}\|\\
&\le\|(\ma{A}_c^{(n)})^{\T} (\wtd{\ma{X}}_{\opt} -\ma{X}_{\opt})\ma{R}^{-1}\| + \|(\ma{A}_c^{(n)})^{\T} \ma{X}_{\opt}\ma{R}^{-1}\|_F\\
&\stackrel{(j)}{\le} \|\wtd{\ma{X}}_{\opt} -\ma{X}_{\opt}\| \|\ma{R}^{-1}\|\\
&\stackrel{(k)}{\le} \frac{\frac{6g_{\max}(C_1+C_2)}{g_{\min}}\sqrt{\frac{\alpha}{p}} \sin\theta_{t-1} }{1- \frac{\sqrt{2}\kappa_n}{g_{\min}} \psi_{t-1}
-  \frac{6g_{\max}(C_1+C_2)}{g_{\min}}\sqrt{\frac{\alpha}{p}} \sin\theta_{t-1}},
\end{align*}
where (i) uses Lemma~\ref{lem:sin},
(j) uses $(\ma{A}_c^{(n)})^{\T} \ma{X}_{\opt}=0$ since $\ma{X}_{\opt}=\ma{A}_*^{(n)}[\te{G}_*]_{(n)}\ma{M}_{*,n}^{\T} \te{S}_{(n)}^{\dagger}$, (k) uses Lemma~\ref{lem:xxdif} and \eqref{wtdx}.
The conclusion follows immediately.
\end{proof}

\begin{lemma}\label{lem:full_resid}
Denote $\sin\theta_t=\max_n \|\sin\Theta(\ma{A}_t^{(n)},\ma{A}_*^{(n)})\|$, $g_{\min}=\min_n \sigma_{\min}(\te{G}_{(n)})$,
and
let  $\phi_t$ be the same as in Lemma~\ref{lem:xxcore}.
Assume {\bf A2}.
Then
\begin{align*}
g_{\min} \; \sin\theta_t  \le \phi_t \le \|\te{T}\|_F [(1+\sin\theta_t)^N-1].
\end{align*}
\end{lemma}

\begin{proof}
Since {\bf A2} and $\ma{A}_t^{(n)}$ has orthonormal columns for all $n$, we know that
\begin{align}\label{xopt}
\te{X}_{\opt}=\llbracket \te{G}_*; (\ma{A}_t^{(1)})^{\T}\ma{A}_*^{(1)},\dots,(\ma{A}_t^{(N)})^{\T}\ma{A}_*^{(N)}\rrbracket.
\end{align}

Let $\ma{P}_{*,n}=\ma{A}_*^{(n)}(\ma{A}_*^{(n)})^{\T}$, $\ma{P}_{t,n}=\ma{A}_t^{(n)}(\ma{A}_t^{(n)})^{\T}$ for all $n$.
By Lemma~\ref{lem:sin}, we have
\begin{align}\label{pptheta}
\|\ma{P}_{t,n} - \ma{P}_{*,n}\| \le \sin\theta_t.
\end{align}

For the upper bound of $\phi_t$, by calculations, we have
\begin{align*}
\phi_t &\stackrel{(a)}{=} \|\llbracket \te{G}_*; \ma{P}_{1,t}\ma{A}_*^{(1)},\dots,\ma{P}_{N,t}\ma{A}_*^{(N)}\rrbracket - \llbracket \te{G}_*; \ma{A}^{(1)}_*,\dots,\ma{A}_*^{(N)}\rrbracket\|_F\\
&=\|\llbracket \te{T}; \ma{P}_{*,1} +(\ma{P}_{1,t} -\ma{P}_{*,1}),\dots,\ma{P}_{*,N} +(\ma{P}_{N,t} -\ma{P}_{*,N})\rrbracket - \te{T}\|_F\\
&\stackrel{(b)}{\le}  \sum_{n=1}^N
\begin{pmatrix}N\\ n\end{pmatrix} \sin\theta^n_t \|\te{T}\|_F
= \|\te{T}\|_F [(1+ \sin\theta_t)^N-1]
\end{align*}
where (a) uses {\bf A2} and \eqref{xopt},
(b) uses \eqref{pptheta} and {\bf A2}.

For the lower bound of $\phi_t$, let us first
fix all $\ma{P}_{t,n}$ except one, say $\ma{P}_{t,1}$.
Unfolding
\[
\llbracket \te{G}_*; \ma{P}_{t,1} \ma{A}_*^{(1)},\dots,\ma{P}_{t,N}\ma{A}_*^{(N)}\rrbracket - \llbracket \te{G}_*; \ma{A}_*^{(1)},\dots,\ma{A}_*^{(N)}\rrbracket
\]
 along mode-1, we get
\begin{align*}
\ma{P}_{t,1}\ma{A}_*^{(1)} \te{T}_{(n)} \ma{P}_{t,N}\ma{A}_*^{(N)}\otimes \dots\otimes \ma{P}_{t,2}\ma{A}_*^{(2)}
-\ma{A}_*^{(1)}\te{T}_{(n)} \ma{A}_*^{(N)}\otimes \dots\otimes \ma{A}_*^{(2)},
\end{align*}
whose Frobenius norm is minimized at
\begin{align}
\ma{P}_{t,1}=\ma{A}_*^{(1)}\te{T}_{(n)} \ma{A}_*^{(N)}\otimes \dots\otimes \ma{A}_*^{(2)} \big(\ma{A}_*^{(1)} \te{T}_{(n)} \ma{P}_{t,N}\ma{A}_*^{(N)}\otimes \dots\otimes \ma{P}_{t,2}\ma{A}_*^{(2)}\big)^{\dagger}.
\end{align}
Using {\bf A4} and the definition of $\ma{P}_{1,t}$, we have $\mathcal{R}(\ma{A}_t^{(1)})=\mathcal{R}(\ma{A}_*^{(1)})$.
Without loss of generality, let $\theta_{1,t}=\theta_t$.
By calculaitons, we have
\begin{align*}
\phi_t &\ge \|\llbracket \te{G}_*; \ma{P}_{1,t}\ma{A}_*^{(1)},\ma{A}_*^{(2)},\dots,\ma{A}_*^{(N)}\rrbracket - \llbracket \te{G}_*; \ma{A}_*^{(1)},\dots,\ma{A}_*^{(N)}\rrbracket\|_F\\
&= \| \te{T}\times_1(\ma{P}_{1,t} -\ma{P}_{*,1})\|_F = \|(\ma{P}_{1,t} -\ma{P}_{*,1})\te{T}_{(1)}\|_F\\
&= \| (\ma{P}_{1,t} -\ma{P}_{*,1})\ma{A}_*^{(1)}[\te{G}_*]_{(1)} (\ma{A}_*^{(N)}\otimes \dots\otimes \ma{A}_*^{(2)})^{\T}\|_F
\stackrel{(c)}{=}\| (\ma{P}_{1,t} -\ma{P}_{*,1})\ma{A}^{(1)}[\te{G}_*]_{(1)}\|_F\\
&\ge \|(\ma{I} - \ma{P}_{1,t})\ma{A}_*^{(1)}\|_F \; \sigma_{\min}([\te{G}_*]_{(1)})
\stackrel{(d)}{\ge} g_{\min} \sin\theta_t,
\end{align*}
where (c) uses $\ma{A}_*^{(N)}\otimes \dots\otimes \ma{A}_*^{(2)}$ is orthonormal,
(d) uses Lemma~\ref{lem:sin}.
This completes the proof.
\end{proof}

\vspace{0.2in}
Now we are ready to show the main theorem.\\

\noindent{\bf Theorem~4.3.}\;
Follow the notations in Lemma~\ref{lem:xxdif}. 
Assume {\bf A1-A5}, $p\in[4p_*,0.5]$, and
\begin{align*}
\mu=\frac{3}{\sqrt{2}} \frac{\|\te{T}\|_F}{g_{\min}} \frac{[(1+\sin\theta_0)^N-1]}{\sin\theta_0}\times
\frac{6g_{\max}(C_1+C_2)\sqrt{\frac{\alpha}{p}}}{g_{\min}- \sqrt{2}\kappa \psi_0 - 6g_{\max}(C_1+C_2)\sqrt{\frac{\alpha}{p}} \sin\theta_0 }<\frac{\gamma}{7\Gamma}.
\end{align*}
Then
\begin{align*}
\|\llbracket \te{G}_t; \ma{A}_t^{(1)},\dots,\ma{A}_t^{(N)}\rrbracket -\te{T}\|_F
\le \frac{7\mu \Gamma}{\gamma}\|\llbracket \te{G}_{t-1}; \ma{A}_{t-1}^{(1)},\dots,\ma{A}_{t-1}^{(N)}\rrbracket -\te{T}\|_F,\quad \mbox{w.h.p.}
\end{align*}
In other words, Algorithm~\ref{alg:tc} converges to the exact solution at a linear rate, {\it w.h.p}.

\begin{proof}
First, we show $\psi_{t+1}\le \mu\psi_t$ by mathematical induction.

Consider $t=1$.
By Lemma~\ref{lem:tt} and $\mu<1$, we know that
\begin{align}\label{t01}
\sin\theta_1\le
\frac{6g_{\max}(C_1+C_2)\sqrt{\frac{\alpha}{p}}}{g_{\min}- \sqrt{2}\kappa \psi_0 - 6g_{\max}(C_1+C_2)\sqrt{\frac{\alpha}{p}} \sin\theta_0 }\sin\theta_0<\sin\theta_0.
\end{align}
Then
\begin{align}
\psi_1 \stackrel{(a)}{\le} \frac{3}{\sqrt{2}}\phi_1
\stackrel{(b)}{\le} \frac{3}{\sqrt{2}} \|\te{T}\|_F [(1+\sin\theta_1)^N-1]
\stackrel{(c)}{\le} \frac{3}{\sqrt{2}} \|\te{T}\|_F \frac{[(1+\sin\theta_0)^N-1]}{\sin\theta_0} \sin\theta_1
\stackrel{(d)}{\le} \mu  g_{\min}\sin\theta_0
\stackrel{(e)}{\le} \mu \psi_0
\stackrel{(f)}{<}\psi_0,\label{psi01}
\end{align}
where (a) uses Lemma~\ref{lem:xxcore}, (b) (e) uses Lemma~\ref{lem:full_resid},
(c) uses $\sin\theta_1<\sin\theta_0$ (by \eqref{t01} and $\mu<1$),
(d) uses \eqref{t01}, (f) uses $\mu<1$.

Now assume $\sin\theta_t<\sin\theta_{t-1}$, $\psi_t<\mu \psi_{t-1}$ for $t\le T$.
Then together with Lemma~\ref{lem:tt} and $\mu<1$, we have
\begin{align*}
\sin\theta_{T+1}
&\le \frac{6g_{\max}(C_1+C_2)\sqrt{\frac{\alpha}{p}}}{g_{\min}- \sqrt{2}\kappa \psi_T - 6g_{\max}(C_1+C_2)\sqrt{\frac{\alpha}{p}} \sin\theta_T }\sin\theta_T\\
&\le \frac{6g_{\max}(C_1+C_2)\sqrt{\frac{\alpha}{p}}}{g_{\min}- \sqrt{2}\kappa \psi_0 - 6g_{\max}(C_1+C_2)\sqrt{\frac{\alpha}{p}} \sin\theta_0 }\sin\theta_T
<\sin\theta_T.
\end{align*}
And similar to the proof of \eqref{psi01}, we get $\psi_{T+1}\le \mu \psi_T$.

Second,
let
$\wtd{\te{X}}_{\opt} =\argmin\|\PO(\llbracket {\te{X}}; \ma{A}_t^{(1)},\dots,\ma{A}_t^{(N)}\rrbracket -\te{T})\|_F$.
Using Lemma~\ref{lem:vecp} and {\bf A4}, we have {\it w.h.p}. that
\begin{align*}
\|\llbracket \te{G}_t; \ma{A}_t^{(1)},\dots,\ma{A}_t^{(N)}\rrbracket -\te{T}\|_F
&\le \frac{2}{\sqrt{p}} \tau_t
\le \frac{2}{\sqrt{p}}  \frac{1}{\gamma}
\|\PO(\llbracket \wtd{\te{X}}_{\opt}; \ma{A}_t^{(1)},\dots,\ma{A}_t^{(N)}\rrbracket)-\PO(\te{T})\|_F\\
&\le \frac{2}{\sqrt{p}} \frac{1}{\gamma}  \frac{\sqrt{7p}}{2} \|\llbracket \wtd{\te{X}}_{\opt}; \ma{A}_t^{(1)},\dots,\ma{A}_t^{(N)}\rrbracket -\te{T}\|_F
= \frac{\sqrt{7}}{\gamma}\psi_t,\\
\|\llbracket \te{G}_t; \ma{A}_t^{(1)},\dots,\ma{A}_t^{(N)}\rrbracket -\te{T}\|_F
&\ge \frac{2}{\sqrt{7p}} \tau_t
\ge \frac{2}{\sqrt{7p}}  \frac{1}{\Gamma}
\|\PO(\llbracket \wtd{\te{X}}_{\opt}; \ma{A}_t^{(1)},\dots,\ma{A}_t^{(N)}\rrbracket)-\PO(\te{T})\|_F\\
&\ge \frac{2}{\sqrt{7p}} \frac{1}{\Gamma}  \frac{\sqrt{p}}{2} \|\llbracket \wtd{\te{X}}_{\opt}; \ma{A}_t^{(1)},\dots,\ma{A}_t^{(N)}\rrbracket -\te{T}\|_F
= \frac{1}{\sqrt{7} \Gamma}\psi_t,
\end{align*}
Combining them with $\psi_{t+1}\le \mu\psi_t$, we get
\begin{align*}
\|\llbracket \te{G}_t; \ma{A}_t^{(1)},\dots,\ma{A}_t^{(N)}\rrbracket -\te{T}\|_F
\le\frac{\sqrt{7}}{\gamma}\psi_t
\le\frac{\sqrt{7}}{\gamma}\mu\psi_{t-1}
\le \frac{7\mu\Gamma}{\gamma}\|\llbracket \te{G}_{t-1}; \ma{A}_t^{(1)},\dots,\ma{A}_t^{(N)}\rrbracket -\te{T}\|_F.
\end{align*}
This completes the proof.
\end{proof}

\section{Assumption (11) in Lemma~4.1}
In this section, we first give a lemma,
then give Proposition~\ref{prop:sub}, a proof for assumption (11) in Lemma~4.1.
\begin{lemma} \label{lem:kron}
Let $\ma{A}_k$, $\wht{\ma{A}}_k\in\R^{I_n\times r_n}$ have orthonormal columns for $k=1,\dots,N$, and denote
\begin{align*}
\ma{M}=\ma{A}_1\otimes \ma{A}_2\otimes \dots\otimes\ma{A}_N,
\quad \wht{\ma{M}}=\wht{\ma{A}}_1\otimes \wht{\ma{A}}_2
\otimes \dots\otimes \wht{\ma{A}}_N.
\end{align*}
Then
\[
\|\sin\Theta(\ma{M},\wht{\ma{M}})\|\le 2^{\frac{N-1}{2}} \max_k \{ \|\sin\Theta(\ma{A}_k,\wht{\ma{A}}_k)\|\}.
\]
\end{lemma}

\begin{proof}
Consider $N=2$. Without loss of generality, let $\|\sin\Theta(\ma{A}_1,\wht{\ma{A}}_1)\| \ge \|\sin\Theta(\ma{A}_2,\wht{\ma{A}}_2)\|$.
Then by the definition of principal angle, it holds
\begin{align}\label{sigaa}
\sigma_{\min}(\ma{A}_1^{\T}\wht{\ma{A}}_1)\ge \sigma_{\min}(\ma{A}_2^{\T}\wht{\ma{A}}_2).
\end{align}
Noticing that
\begin{align}\label{sigmm}
\sigma_{\min}(\ma{M}^{\T}\wht{\ma{M}})
=\sigma_{\min}(\ma{A}_1^{\T}\wht{\ma{A}}_1 \otimes \ma{A}_2^{\T}\wht{\ma{A}}_2)
=\sigma_{\min}(\ma{A}_1^{\T}\wht{\ma{A}}_1 )\sigma_{\min} (\ma{A}_2^{\T}\wht{\ma{A}}_2),
\end{align}
by calculations, we have
\begin{align*}
\|\sin\Theta(\ma{M},\wht{\ma{M}})\|
&\stackrel{(a)}{=} \sqrt{1-\sigma_{\min}^2(\ma{M}^{\T}\wht{\ma{M}})}
\stackrel{(b)}{=} \sqrt{1-\sigma_{\min}^2(\ma{A}_1^{\T}\wht{\ma{A}}_1)   \sigma_{\min}^2(\ma{A}_2^{\T}\wht{\ma{A}}_2)}\\
&\stackrel{(c)}{\le} \sqrt{1-\sigma_{\min}^4(\ma{A}_1^{\T}\wht{\ma{A}}_1) }
\le \sqrt{1+\sigma_{\min}^2(\ma{A}_1^{\T}\wht{\ma{A}}_1) } \sqrt{1-\sigma_{\min}^2(\ma{A}_1^{\T}\wht{\ma{A}}_1) }\\
&\le \sqrt{2} \|\sin\Theta(\ma{A}_1,\wht{\ma{A}}_1)\|.
\end{align*}
where (a) uses the definition of principal angle,
(b) uses \eqref{sigmm},
(c) uses \eqref{sigaa}.
The conclusion follows by recursively applying the above result for $N=2$.
\end{proof}

\begin{proposition}\label{prop:sub}
Follow the notations and assumptions in Lemma~\ref{lem:xxcore}.
Then
\begin{align*}
\|\sin\Theta( \ma{M}_{t,t,n} [\wtd{\te{X}}_{\opt}]_{(n)}^{\T},  \ma{M}_{*,n}[\te{G}_*]_{(n)}^{\T})\|\le C\sin\theta_t,
\end{align*}
where
\[
C= \frac{2^{\frac{N-2}{2}} g_{\max} }{g_{\min} - 2^{\frac{N-2}{2}} g_{\max} \sin\theta_t}
+ \frac{\frac{6\sqrt{\alpha J_n}}{\sqrt{p}} g_{\max} \big(
\prod_{k\ne n} \sqrt{\frac{\mu_k r_k}{I_k}} + \sqrt{\frac{\mu_n r_n}{I_n}} 2^{\frac{N-2}{2}}
\big)}
{(1-\frac{6\sqrt{\alpha J_n}}{\sqrt{p}}  \sqrt{\frac{\mu_n r_n}{I_n}} )(g_{\min} - 2^{\frac{N-2}{2}} g_{\max} \sin\theta_t)}.
\]

\end{proposition}

\begin{proof}

First, we show an upper bound for $\|\sin\Theta( \ma{M}_{t,t,n} [\te{X}_{\opt}]_{(n)}^{\T},  \ma{M}_{*,n}[\te{G}_*]_{(n)}^{\T})\|$.

Recall the definition of $\te{X}_{\opt}$, we have
\begin{align}\label{xoptn}
[\te{X}_{\opt}]_{(n)}  =  (\ma{A}_t^{(n)})^{\T}\ma{A}_*^{(n)} [\te{G}_*]_{(n)} \ma{M}_{*,n}^{\T} \ma{M}_{t,t,n}.
\end{align}
By calculations, we get
\begin{align}
 &\|\ma{M}_{t,t,n} [\te{X}_{\opt}]_{(n)}^{\T} -  \ma{M}_{*,n} [\te{G}_*]_{(n)}^{\T} (\ma{A}_*^{(n)})^{\T} \ma{A}_t^{(n)}\|\notag\\
 \stackrel{(a)}{=} &\|(\ma{I} -\ma{M}_{t,t,n} \ma{M}_{t,t,n}^{\T} ) \ma{M}_{*,n} [\te{G}_*]_{(n)}^{\T} (\ma{A}_*^{(n)})^{\T} \ma{A}_t^{(n)}\|\notag\\
 \stackrel{(b)}\le & \|\sin\Theta(\ma{M}_{t,t,n},\ma{M}_{*,n})\| g_{\max} \cos\theta_t
 \stackrel{(c)}{\le}  2^{\frac{N-2}{2}} g_{\max} \sin\theta_t \cos\theta_t,\label{mxmg}
\end{align}
where (a) uses \eqref{xoptn}, (b) uses Lemma~\ref{lem:sin},
(c) uses Lemma~\ref{lem:kron}.
It follows that
\begin{align}
\sigma_{\min}(\ma{M}_{t,t,n} [\te{X}_{\opt}]_{(n)}^{\T})
&\stackrel{(d)}{\ge} \sigma_{\min}(\ma{M}_{*,n} [\te{G}_*]_{(n)}^{\T} (\ma{A}_*^{(n)})^{\T} \ma{A}_t^{(n)}) - 2^{\frac{N-2}{2}} g_{\max} \sin\theta_t\cos\theta_t\notag\\
& = \sigma_{\min}([\te{G}_*]_{(n)}^{\T}  (\ma{A}_*^{(n)})^{\T} \ma{A}_t^{(n)}) - 2^{\frac{N-2}{2}} g_{\max} \sin\theta_t \cos\theta_t\notag\\
&\ge g_{\min} \cos\theta_t - 2^{\frac{N-2}{2}} g_{\max} \sin\theta_t \cos\theta_t.\label{sigmx}
\end{align}
where (d) \eqref{mxmg}.
Then we get
\begin{align}
&\|\sin\Theta( \ma{M}_{t,t,n} [\te{X}_{\opt}]_{(n)}^{\T},  \ma{M}_{*,n}[\te{G}_*]_{(n)}^{\T})\|\notag\\
\stackrel{(e)}{\le}& \frac{\|(\ma{I}-\ma{M}_{*,n}[\te{G}_*]_{(n)}^{\T}(\ma{M}_{*,n}[\te{G}_*]_{(n)}^{\T})^{\dagger}) \ma{M}_{t,t,n} [\te{X}_{\opt}]_{(n)}^{\T}\|}{\sigma_{\min}(\ma{M}_{t,t,n} [\te{X}_{\opt}]_{(n)}^{\T})}\notag\\
\stackrel{(f)}\le &\frac{\|(\ma{I}-\ma{M}_{*,n}[\te{G}_*]_{(n)}^{\T}(\ma{M}_{*,n}[\te{G}_*]_{(n)}^{\T})^{\dagger}) (\ma{M}_{t,t,n} [\te{X}_{\opt}]_{(n)}^{\T} -   \ma{M}_{*,n} [\te{G}_*]_{(n)}^{\T} (\ma{A}_*^{(n)})^{\T} \ma{A}_t^{(n)})\|}{g_{\min} \cos\theta_t - 2^{\frac{N-2}{2}} g_{\max} \sin\theta_t \cos\theta_t}\notag\\
\stackrel{(g)}{\le} & \frac{2^{\frac{N-2}{2}} g_{\max} \sin\theta_t }{g_{\min} - 2^{\frac{N-2}{2}} g_{\max} \sin\theta_t},\label{sin1}
\end{align}
where (e) uses Lemma~\ref{lem:sin}, (f) uses \eqref{sigmx}, and (g) uses \eqref{mxmg}.

\vspace{0.1in}

Next, we show an upper bound for $\|\sin\Theta( \ma{M}_{t,t,n} [\wtd{\te{X}}_{\opt}]_{(n)}^{\T},  \ma{M}_{t,t,n} [\te{X}_{\opt}]_{(n)}^{\T})\|$.

Let $\te{X} = \te{X}_{\opt} +  \te{Y}$. Since
$\wtd{\te{X}}_{\opt}=\argmin\|\llbracket \wtd{\te{X}}_{\opt}; \ma{A}_t^{(1)},\dots,\ma{A}_t^{(N)}\rrbracket -\te{T}\|_F$,
Let $\te{Y}$ be the least square solution to
\begin{align}\label{poar}
\frac1p\PO(\ma{A}_t^{(n)} \te{Y}_{(n)} \ma{M}_{t,t,n}^{\T}) =\frac1p \PO(\ma{R}),
\end{align}
where $\ma{R}= \ma{A}_*^{(n)} [\te{G}_*]_{(n)} \ma{M}_{*,n}^{\T} - \ma{A}_t^{(n)}  \te{X}_{\opt} \ma{M}_{t,t,n}^{\T}$,
then $\wtd{\te{X}}_{\opt} = \te{X}_{\opt} +  \te{Y}$.

\vspace{0.1in}

By Lemma~\ref{lem:prob}, {\it w.h.p}., there exist $\ma{E}$ and $\ma{F}$ such that
\begin{align}
\frac1p\PO(\ma{A}_t^{(n)} \te{Y}_{(n)} \ma{M}_{t,t,n}^{\T})=\ma{A}_t^{(n)} \te{Y}_{(n)} \ma{M}_{t,t,n}^{\T} + \ma{E},
\quad
\frac1p \PO(\ma{R})= \ma{R} +\ma{F},
\end{align}
where
\begin{align}
\|\ma{E}\|&\le \frac{6\sqrt{\alpha J_n}}{\sqrt{p}} \|\ma{A}_t^{(n)} \te{Y}_{(n)} \ma{M}_{t,t,n}^{\T}\|_{\max}
\stackrel{(h)}{\le} \frac{6\sqrt{\alpha J_n}}{\sqrt{p}}  \sqrt{\frac{\mu_n r_n}{I_n}}  \|\te{Y}_{(n)}\ma{M}_{t,t,n}^{\T}\|,\label{e2}
\\
\|\ma{F}\|&\le \frac{6\sqrt{\alpha J_n}}{\sqrt{p}} \|\ma{R}\|_{\max}\notag\\
&\le \frac{6\sqrt{\alpha J_n}}{\sqrt{p}} \big(
\|\ma{A}_*^{(n)} [\te{G}_*]_{(n)} \ma{M}_{*,n}^{\T} - \ma{A}_t^{(n)}  (\ma{A}_t^{(n)})^{\T}\ma{A}_*^{(n)} [\te{G}_*]_{(n)} \ma{M}_{*,n}^{\T}\|_{\max}
+\notag\\
&\hskip0.2in \|\ma{A}_t^{(n)}  (\ma{A}_t^{(n)})^{\T}\ma{A}_*^{(n)} [\te{G}_*]_{(n)} \ma{M}_{*,n}^{\T}- \ma{A}_t^{(n)}  (\ma{A}_t^{(n)})^{\T}\ma{A}_*^{(n)} [\te{G}_*]_{(n)} \ma{M}_{*,n}^{\T} \ma{M}_{t,t,n} \ma{M}_{t,t,n}^{\T}\|_{\max}\big)
\notag\\
&\stackrel{(i)}\le \frac{6\sqrt{\alpha J_n}}{\sqrt{p}} g_{\max} \big(
\prod_{k\ne n} \sqrt{\frac{\mu_k r_k}{I_k}} \sin\theta_t + \sqrt{\frac{\mu_n r_n}{I_n}}
\|\sin\Theta(\ma{M}_{t,t,n},\ma{M}_{*,n})\|
\big)\notag\\
&\stackrel{(j)}{\le} \frac{6\sqrt{\alpha J_n}}{\sqrt{p}} g_{\max} \big(
\prod_{k\ne n} \sqrt{\frac{\mu_k r_k}{I_k}} + \sqrt{\frac{\mu_n r_n}{I_n}} 2^{\frac{N-2}{2}}
\big)\sin\theta_t,\label{f2}
\end{align}
where (h) uses {\bf A3}, (i) uses {\bf A3}, {\bf A3}, Lemma~\ref{lem:sin}, and (j) uses Lemma~\ref{lem:kron}.

\vspace{0.1in}

Then we can rewrite \eqref{poar} as $\ma{A}_t^{(n)} \te{Y}_{(n)} \ma{M}_{t,t,n}^{\T} + \ma{E} = \ma{R} +\ma{F}$,
from which we can obtain
\begin{align}\label{teyn}
\te{Y}_{(n)} = (\ma{A}_t^{(n)})^{\T}(\ma{R}+\ma{F}-\ma{E})\ma{M}_{t,t,n} = (\ma{A}_t^{(n)})^{\T}(\ma{F}-\ma{E})\ma{M}_{t,t,n}.
\end{align}
It follows from \eqref{e2}, \eqref{f2} and \eqref{teyn} that
\begin{align}
\|\te{Y}_{(n)} \ma{M}_{t,t,n}^{\T}\|
\le \frac{1}{1-\frac{6\sqrt{\alpha J_n}}{\sqrt{p}}  \sqrt{\frac{\mu_n r_n}{I_n}} }\|\ma{F}\|
\le \frac{\frac{6\sqrt{\alpha J_n}}{\sqrt{p}} g_{\max} \big(
\prod_{k\ne n} \sqrt{\frac{\mu_k r_k}{I_k}} + \sqrt{\frac{\mu_n r_n}{I_n}} 2^{\frac{N-2}{2}}
\big)}{1-\frac{6\sqrt{\alpha J_n}}{\sqrt{p}}  \sqrt{\frac{\mu_n r_n}{I_n}} } \sin\theta_t.\label{ymttn}
\end{align}

Then we get
\begin{align}
&\|\sin\Theta( \ma{M}_{t,t,n} [\wtd{\te{X}}_{\opt}]_{(n)}^{\T},  \ma{M}_{t,t,n} [\te{X}_{\opt}]_{(n)}^{\T})\|\notag\\
\stackrel{(k)}{\le}& \frac{\|(\ma{I}-\ma{M}_{t,t,n} [\wtd{\te{X}}_{\opt}]_{(n)}^{\T}(\ma{M}_{t,t,n} [\wtd{\te{X}}_{\opt}]_{(n)}^{\T})^{\dagger}) \ma{M}_{t,t,n} [\te{X}_{\opt}]_{(n)}^{\T}\|}{\sigma_{\min}(\ma{M}_{t,t,n} [\te{X}_{\opt}]_{(n)}^{\T})}\notag\\
\stackrel{(l)}{\le} &\frac{\|(\ma{I}-\ma{M}_{t,t,n} [\wtd{\te{X}}_{\opt}]_{(n)}^{\T}(\ma{M}_{t,t,n} [\wtd{\te{X}}_{\opt}]_{(n)}^{\T})^{\dagger}) (\ma{M}_{t,t,n} [\te{X}_{\opt}]_{(n)}^{\T} -   \ma{M}_{t,t,n} [\wtd{\te{X}}_{\opt}]_{(n)}^{\T})\|}{g_{\min} \cos\theta_t - 2^{\frac{N-2}{2}} g_{\max} \sin\theta_t \cos\theta_t}\notag\\
\stackrel{(m)}\le & \frac{\frac{6\sqrt{\alpha J_n}}{\sqrt{p}} g_{\max} \big(
\prod_{k\ne n} \sqrt{\frac{\mu_k r_k}{I_k}} + \sqrt{\frac{\mu_n r_n}{I_n}} 2^{\frac{N-2}{2}}
\big)\sin\theta_t}
{(1-\frac{6\sqrt{\alpha J_n}}{\sqrt{p}}  \sqrt{\frac{\mu_n r_n}{I_n}} )(g_{\min} - 2^{\frac{N-2}{2}} g_{\max} \sin\theta_t)},\label{sin2}
\end{align}
where (k) uses Lemma~\ref{lem:sin},
(l) uses~\ref{mxmg}, and
(m) uses \eqref{ymttn}.

\vspace{0.1in}

Finally, combining \eqref{sin1} and \eqref{sin2}, we obtain
\begin{align*}
&\|\sin\Theta( \ma{M}_{t,t,n} [\wtd{\te{X}}_{\opt}]_{(n)}^{\T},  \ma{M}_{*,n}[\te{G}_*]_{(n)}^{\T})\|\\
\le &\|\sin\Theta( \ma{M}_{t,t,n} [\wtd{\te{X}}_{\opt}]_{(n)}^{\T},  \ma{M}_{t,t,n} [\te{X}_{\opt}]_{(n)}^{\T})\|+ \|\sin\Theta( \ma{M}_{t,t,n} [\te{X}_{\opt}]_{(n)}^{\T},  \ma{M}_{*,n}[\te{G}_*]_{(n)}^{\T})\|\\
\le & \frac{2^{\frac{N-2}{2}} g_{\max} \sin\theta_t }{g_{\min} - 2^{\frac{N-2}{2}} g_{\max} \sin\theta_t}
+ \frac{\frac{6\sqrt{\alpha J_n}}{\sqrt{p}} g_{\max} \big(
\prod_{k\ne n} \sqrt{\frac{\mu_k r_k}{I_k}} + \sqrt{\frac{\mu_n r_n}{I_n}} 2^{\frac{N-2}{2}}
\big)\sin\theta_t}
{(1-\frac{6\sqrt{\alpha J_n}}{\sqrt{p}}  \sqrt{\frac{\mu_n r_n}{I_n}} )(g_{\min} - 2^{\frac{N-2}{2}} g_{\max} \sin\theta_t)},
\end{align*}
completing the proof.
\end{proof}

\end{document}